\documentclass{article}

\usepackage{microtype}
\usepackage{graphicx}
\usepackage{subfigure}
\usepackage{booktabs} 

\usepackage{hyperref}

\usepackage[accepted]{icml2023}

\usepackage{amsmath}
\usepackage{amssymb}
\usepackage{mathtools}
\usepackage{amsthm}

\usepackage[capitalize,noabbrev]{cleveref}

\theoremstyle{plain}
\newtheorem{theorem}{Theorem}[section]
\newtheorem{proposition}[theorem]{Proposition}
\newtheorem{lemma}[theorem]{Lemma}

\theoremstyle{definition}

\newtheorem{assumption}[theorem]{Assumption}
\theoremstyle{remark}

\usepackage[textsize=tiny]{todonotes}

\icmltitlerunning{Safe Offline Reinforcement Learning with Real-Time Budget Constraints}

\begin{document}

\twocolumn[
\icmltitle{Safe Offline Reinforcement Learning with Real-Time Budget Constraints}



\icmlsetsymbol{equal}{*}

\begin{icmlauthorlist}
\icmlauthor{Qian Lin}{sch}
\icmlauthor{Bo Tang}{comp,equal}
\icmlauthor{Zifan Wu}{sch,equal}
\icmlauthor{Chao Yu}{sch}
\icmlauthor{Shangqin Mao}{comp}
\icmlauthor{Qianlong Xie}{comp}
\icmlauthor{Xingxing Wang}{comp}
\icmlauthor{Dong Wang}{comp}
\end{icmlauthorlist}

\icmlaffiliation{sch}{Sun Yat-Sen University}
\icmlaffiliation{comp}{Meituan}

\icmlcorrespondingauthor{Chao Yu}{yuchao3@mail.sysu.edu.cn}

\icmlkeywords{Machine Learning, ICML}

\vskip 0.3in
]



\printAffiliationsAndNotice{\icmlEqualContribution} 

\begin{abstract}
Aiming at promoting the safe real-world deployment of Reinforcement Learning (RL), research on safe RL has made significant progress in recent years.
However, most existing works in the literature still focus on the online setting where risky violations of the safety budget are likely to be incurred during training.
Besides, in many real-world applications, the learned policy is required to respond to dynamically determined safety budgets (i.e., constraint threshold) in real time.
In this paper, we target at the above \emph{real-time budget constraint problem} under the offline setting, and propose \textbf{T}rajectory-based \textbf{RE}al-time \textbf{B}udget \textbf{I}nference (\textbf{TREBI}) as a novel solution that models this problem from the perspective of trajectory distribution and solves it through diffusion model planning.
Theoretically, we prove an error bound of the estimation on the episodic reward and cost under the offline setting and thus provide a performance guarantee for TREBI.
Empirical results on a wide range of simulation tasks and a real-world large-scale advertising application demonstrate the capability of TREBI in solving real-time budget constraint problems under offline settings.
\end{abstract}

\section{Introduction}

In recent years, Reinforcement Learning (RL)~\cite{sutton2018reinforcement} has achieved great successes in solving complex decision-making problems, such as games~\cite{silver2016mastering,vinyals2019grandmaster}, robotics~\cite{peng2018sim,hanna2021grounded} and recommendation systems~\cite{zheng2018drn}. 
However, the concern of safety still remains a major challenge preventing many real-world deployments of RL, which motivates the research of safe RL~\cite{xu2022trustworthy,gu2022review}.
Briefly speaking, safe RL aims to learn a policy that maximizes the long-term reward while satisfying certain constraints.
Many safe RL approaches have been proposed in the past few years~\cite{achiam2017constrained,zhang2020first,sootla2022saute,liu2022constrained}.
Unfortunately, most existing approaches only target at the online setting, where potentially risky constraint
violations can be incurred during interactions with the real environment.
As a kind of data-driven methods, offline RL~\cite{levine2020offline} aims to derive a policy from offline data without further real-world exploration, and thus is particularly suitable for safety-critical applications.
Despite the recent progress in the offline RL literature~\cite{fujimoto2019off,kumar2020conservative,fujimoto2021minimalist}, however, there are still limited works focusing on attaining a safe policy under the offline setting.


Moreover, many real-world scenarios require the deployed policy to respond to dynamically determined budgets (i.e., constraint threshold) in real time, rather than simply satisfy the same budget fixed throughout both the training and the deployment phase.
In such a \emph{real-time budget constraint problem}, it is often impossible to retrain a policy from scratch for each unseen budget during the deployment of the learned policy. 
Taking the advertising scenario for example, different advertisers set different budgets in real time, and the advertising system needs to make the optimal decision on the allocation of display channels under different budgets.
Also, in the autonomous driving scenario, the road conditions could change rapidly, imposing time-varying constraints on the speed, acceleration or safe distance of the vehicles, hence the policy should be able to quickly adapt to such changes without further time-consuming retraining.


In this paper, we provide a novel offline RL solution, i.e., the \textbf{T}rajectory-based \textbf{RE}al-time \textbf{B}udget \textbf{I}nference (\textbf{TREBI}), to the above real-time budget constraint problem.
The key idea of TREBI lies in approaching the constrained policy optimization problem from the perspective of trajectory distribution. 
This can not only yield the derivation of the optimal trajectory distribution w.r.t. a certain budget and a given offline dataset, but also enable strict constraint satisfaction (i.e., with probability one in theory) in the trajectory level, which is unlike the rough constraint satisfaction (i.e., expectation over trajectories) in most of previous studies.
Note that this is of great significance in real-life applications of RL, particularly for safety-critical domains where only one single constraint violation could possibly lead to catastrophic consequences.
Moreover, TREBI requires only a mild assumption regarding the quality of offline data. Thus, it is able to learn a safe policy even from a dataset collected by an unconstrained policy.


To model the optimal trajectory distribution w.r.t. a certain budget, we adopt the Diffuser~\cite{janner2022planning}, a trajectory optimization method based on the diffusion model~\cite{sohl2015deep,ho2020denoising}, as the backbone for TREBI.
Specifically, the trajectory distribution of the behavior policy is approximated using the offline dataset in the learning phase, and then adaptive responses to the real-time budgets are achieved by employing a budget-related trajectory planning in the inference phase.
Theoretically, we prove an error bound of the offline estimation on the episodic reward and cost, which in turn yields a performance guarantee for TREBI.
Empirically, we demonstrate that TREBI can solve the offline real-time budget constraint problem effectively 
compared to the existing safe offline RL baselines both in simulation and real-world tasks.
The capability to utilize datasets with various levels of quality is also verified in the experiments.

\section{Related Work}

\textbf{Offline RL}
\quad Offline RL refers to the problem of learning RL policies from static datasets. 
The main issue of offline RL is the distribution shift between the state-action pairs in the dataset and those induced by the learned policy~\cite{fujimoto2019off}. 
To address this issue, model-free algorithms generally propose to regularize the learned policy to stay close to the behavior policy~\cite{fujimoto2019off,kumar2019stabilizing,wu2019behavior}, learn conservative Q-values~\cite{kumar2020conservative,lyu2022mildly}, while many model-based algorithms adopt uncertainty estimation~\cite{yu2020mopo,kidambi2020morel} or conservative value function estimation~\cite{yu2021combo, rigter2022rambo}.

\textbf{Safe RL}
\quad One of the most commonly adopted paradigms for solving safe RL problems is alternating between optimizing the policy and updating the Lagrange dual variables, which is referred to as the primal-dual approaches~\cite{chow2017risk,achiam2017constrained,tessler2018reward,ding2020natural,stooke2020responsive,as2022constrained}.
However, this kind of approaches can be sensitive to the initialization and the learning rate of Lagrange multipliers.
On the contrary, primal approaches~\cite{xu2021crpo, sootla2022saute} sidestep this issue by excluding the dual variables from the optimization framework.
Falling into this category, TREBI also does not need to optimize any dual variables.

\textbf{Safe Offline RL}
\quad Despite the significant progress of offline RL in recent years, works focusing on the safe offline setting are still limited~\cite{le2019batch,zhang2021bcorle,xu2022constraints,lee2022coptidice}.
CBPL~\cite{le2019batch} solves the primal-dual problem from the game-theoretic perspective where the training and evaluation processes are performed alternatively and Lagrange multipliers are updated to narrow the dual gap. 
CPQ~\cite{xu2022constraints} applies pessimistic estimation of cost value functions for out-of-distribution actions, and updates the reward value function as well as the policy only w.r.t. safe actions. 
Rather than using a policy or value functions like most of previous works, COptiDICE~\cite{lee2022coptidice} directly optimizes the state-action stationary distribution and then extracts the policy from the stationary distribution by importance-weighted behavioral cloning. 
BCORLE($\lambda$) optimizes the policy under  offline BCQ~\cite{fujimoto2019off} and augments the state with Lagrangian multiplier. 
An concurrent work on safe offline RL, CDT~\cite{liu2023constrained}, also focuses on the same problem setting as ours, where the constraint threshold is dynamically determined. 
It combines the decision transformer~\cite{chen2021decision} with data augmentation and solve the problem from a perspective of multi-objective optimization.
As a safe offline method, TREBI differs from the above methods mainly in terms of solving the real-time budget constraint problem in the trajectory level with strict constraint satisfaction.
\textbf{Trajectory Optimization with Planning} \quad Model-based planning methods~\cite{camacho2013model,botev2013cross,nagabandi2020deep} have been widely developed to solve the trajectory optimization problem. 
Such methods are flexible in switching the optimization objective dynamically during planning, and thus suitable for solving the real-time budget constraint problem.
PETS~\cite{chua2018deep} incorporates uncertainty into the dynamics model by using an ensemble of bootstrapped models and compute the optimal trajectory by CEM~\cite{botev2013cross}. 
RCE~\cite{liu2020constrained}, as a model-based planning method for safe RL problems, utilizes a bootstrapped ensemble of neural networks and applies CEM to maximize the expected returns of safe action sequences.
It takes into account the model uncertainty and achieves robustness under this uncertainty.
However, using single-step dynamics to autoregressively generate trajectories may suffer from the compounding rollout errors of long-term predictions~\cite{janner2019trust}, which would further lead to tremendous estimation error of value functions. 
In the safe RL context, such an error can have a detrimental effect on the constraint satisfaction.
On the contrary, TREBI performs planning by generating the whole trajectories simultaneously and thus is able to avoid  the compounding error issue.


\section{Preliminaries}
\subsection{Problem Setting}
The problem of safe RL with a budget can be formulated as a Constrained Markov Decision Process (CMDP), which is represented by a tuple $(\mathcal S,\mathcal A,T,r,c,\gamma,b)$, where
$\mathcal S$ is the state space; $\mathcal A$ is the action space; $T:\mathcal{S}\times \mathcal{A}\times \mathcal{S}\rightarrow [0,1]$ is the transition probability function, i.e., $s_{t+1}\sim T(\cdot|s_t,a_t)$; $r:\mathcal{S}\times \mathcal{A}\rightarrow [-R_{\max},R_{\max}]$ and $c:\mathcal{S}\times \mathcal{A}\rightarrow [0,C_{\max}]$ denote the reward and cost function, bounded by $R_{\max}$ and $C_{\max}$, respectively; $\gamma\in [0,1]$ is the discount factor for future reward and cost; and $b$ is the safety budget. 
A policy $\pi$ : $\mathcal{S}\rightarrow P(\mathcal{A})$ maps the states to a probability distribution over actions. 

A trajectory $\tau$ is composed of a sequence of state-action pairs: $\tau=\{\tau_{s_0}, \tau_{a_0}, \tau_{s_1}, \tau_{a_1}, ...,\tau_{s_L}, \tau_{a_L}\}$, where $L$ is the maximum episode length. 
The cumulative reward and cumulative cost of trajectory $\tau$ are denoted as $R(\tau)=\sum_{t=0}^{L}\gamma^t r(\tau_{s_t},\tau_{a_t})$ and $C(\tau)=\sum_{t=0}^{L}\gamma^t c(\tau_{s_t},\tau_{a_t})$, respectively. 
The policy can be inferred from the distribution of trajectory $p(\tau)$ since the actions can be determined by: $\pi(a|s)\propto\int p(\tau|\tau_{s_0}=s, \tau_{a_0}=a)d\tau_{}$. 
In other words, the trajectory distribution can be seen as a special policy, since the output action $a$ w.r.t. the input state $s$ can be obtained by adopting the first action of the trajectory sampled from the conditioned distribution $p(\tau|\tau_{s_0}=s)$.
%
%
%
%
%
%

In the context of safe offline RL, we assume no access to online interactions with the environment and offline dataset $\mathcal{D}=\{(s,a,s',r,c)\}$ is the only data available for training. 
The dataset $\mathcal{D}$ is generated from a set of behavior policies jointly, which might violate the budget constraints.
Let $\pi_\beta(a|s)=\frac{N(s,a)}{N(s)}$ denote the empirical behavior policy and $\hat T(s'|s,a)=\frac{N(s,a,s')}{N(s,a)}$ the empirical probability function, where $N(\cdot)$ is the count of the corresponding tuple observed in $\mathcal{D}$. 
Let $p_{\pi_\beta}$ denote the distribution of trajectories generated by executing $\pi_\beta$ in $\hat T$. 
The goal of safe offline RL with strict constraint satisfaction is to learn a policy $\pi$ from $\mathcal{D}$ that maximizes the expected cumulative reward while satisfying the cumulative cost in the trajectory level:
\begin{align}
    &\max_\pi \mathbb{E}_{\tau\sim \pi,T}[R(\tau)] \nonumber \\  
    \text{s.t. }& \forall{\tau \sim \pi,T}\ \ \  C(\tau)\leq b.   \nonumber
\end{align} 
where $\tau \sim \pi,T$ denotes that $\tau$ is generated by executing $\pi$ in $T$.

\subsection{Real-time Budget Constraint}
\label{sec:real_time_budget_formulation}
The above safe RL problem with a fixed budget can be extended to a real-time budget constraint problem by taking dynamically determined budgets into account. 
The set of candidate budgets $b$ is denoted as $B$ which is bound and can contain finite or infinite elements.
The optimization objective for any given budget $ b\in B$ is then given as follows:
\begin{align}
    &\max_{\pi_b} \mathbb{E}_{\tau\sim \pi_b,T}[R(\tau)]\nonumber\\
    \text{s.t. }& \forall{\tau \sim \pi_b,T}\ \ \  C(\tau)\leq b.   \nonumber
\end{align}
%
where $\pi_b=\pi(\cdot|s,b)$ denotes the policy under the given safety budget $b$. 
To solve the real-time budget problem, a simple and intuitive idea is to discretize the range of budgets and train a single policy for each discrete value of budgets with the off-the-shelf fixed-budget algorithms~\cite{xu2022constraints,lee2022coptidice}. 
However, it is intractable to determine the appropriate granularity of budget division without guidance of prior knowledge. 
A small granularity leads to a large number of candidate policies to train, which is computationally expensive, while a large granularity incurs the difficulty of precisely covering all possible budgets. 
Safe RL with state augmentation~\cite{zhang2021bcorle,sootla2022saute} is another solution, where budget-related information is injected into the state definition to assist policy in dealing with different budget constraints. 
However, the growth of dataset size and state dimension caused by state augmentation may lead to scalability issues and extra computational cost during policy training.
Our solution is to introduce planning based on the trajectory optimization to provide a flexibility for dynamic switching of constrained optimization objectives during the real-time inference process.

\subsection{Diffuser}
\label{sec:diffuser}
In diffusion probalitistic models~\cite{sohl2015deep}, by iteratively adding noise to the data through an forward diffusion process $q(\tau^{i}|\tau^{i-1})$, the raw data $\tau^0$ in the offline dataset is converted into $\tau^N$ that approximately conforms to a noise distribution (e.g., the standard Gaussian).
The denoising process $p(\tau^{i-1}|\tau^{i})$, as the reverse of $q(\tau^{i}|\tau^{i-1})$, can be utilized to progressively rebuild the data distribution from the noise distribution.
The reverse process is often parameterized as Gaussian, i.e., $p(\tau^{i-1}|\tau^{i})=\mathcal{N}(\tau^{i-1};\mu_\theta(\tau^i,i),\Sigma^i)$, while the forward process is pre-determined.

Trajectory optimization in Diffuser is divided into two phases. 
During training, a denoising model $p_\theta(\tau^{i-1}|\tau^{i})$ is trained to model the trajectory distribution of behavior policy $p_{\pi_\beta}$ by minimizing a variational bound on the negative log likelihood: $\min_{\theta}-E_{\tau\sim D}[\log p_{\theta}(\tau)]$, i.e., minimizing the KL divergence between $p_\theta$ and $p_{\pi_\beta}$.
Meanwhile, a function $h(\tau)$ is learned to indicate the optimization objective for $\tau$.
For instance, to maximize the cumulative reward of generated trajectories, $h(\tau)$ can be set as the exponential of reward value estimation fitted on the offline trajectories.
During inference, the generated trajectories are guided for planning to achieve the optimization objective.
Specifically, provided that $\log h(\tau)$ is smooth enough~\cite{sohl2015deep}, 
we can sample trajectories from $\tilde p_{\theta}(\tau) = p_\theta(\tau)h(\tau)$ by modifying the original reverse process of diffusion model to:
\begin{align}
    \tilde p_{\theta}(\tau^{i-1}|\tau^i)&\propto p_{\theta}(\tau^{i-1}|\tau^i)h(\tau^i)\nonumber\\
    &\approx \mathcal{N}(\tau^{i-1};\mu_\theta(\tau^{i},i)+ \Sigma^i g,\Sigma^i),\label{eq:modified_reverse}
\end{align}
where $\tau^i$ denotes the trajectory obtained in the $i^{th}$ step of reverse process and $g=\nabla_{\tau^i} \log h(\tau^i)$. 

\section{Trajectory-based REal-time Budget Inference (TREBI)}

In this section, we first introduce a trajectory perspective for the safe offline RL in Subsection~\ref{sec:rl_to_match}, and then provide a practical algorithm to this problem in Subsection~\ref{sec:traj_opti}, followed by key theoretical analysis in Subsection~\ref{sec:theoretical}.



\subsection{A Trajectory Optimization Perspective of Safe Offline RL}
\label{sec:rl_to_match}

In essence, trajectory optimization methods~\cite{chen2021decision,janner2021offline,furuta2021generalized,janner2022planning} fall into the model-based RL category by simultaneously incorporating the dynamics and the policy into one single optimization objective.
In the context of offline RL, the model-based objective of policy $\pi$ can be written as maximizing the expected return under the empirical dynamics $\hat{T}$, i.e., $\max_{\pi} \mathbb{E}_{\tau\sim \pi,\hat T}[R(\tau)]$.
To alleviate the distribution shift issue, a common practice in offline RL is to constrain the divergence between the target policy and the behavior policy~\cite{kumar2019stabilizing,wu2019behavior,fujimoto2021minimalist}, which is referred to as the \emph{in-distribution constraint} in this paper.
From the perspective of trajectory optimization, the in-distribution constraint can be expressed in the following form:  $D_{KL}(q(\tau)||p_{\pi_\beta}(\tau))\leq \epsilon$,
where $q(\tau)$is target trajectory distribution and $\epsilon$ is an approximately chosen divergence threshold.
Then, the policy optimization problem can be converted to the form of trajectory optimization as follows:
\begin{align}
    &\max_{q(\tau)} \mathbb{E}_{\tau\sim q(\tau)}[R(\tau)], \label{eq:traj_oj}\\
    \text{s.t. } &D_{KL}(q(\tau)||p_{\pi_\beta}(\tau))\leq \epsilon.\label{eq:traj_con}
\end{align}
It is worth noting that such a conversion is equivalent only when the dynamics induced by the optimized $q$ is consistent with $\hat{T}$.
In environments with deterministic dynamics, this condition is satisfied since the inconsistency between $\hat{T}$ and the dynamics induced by the optimized $q$ would result in $p_{\pi_{\beta}}(\tau) = 0$ and $q(\tau) \neq 0$ simultaneously for some $\tau$, which would in turn lead to an infinite KL Divergence in the left hand side of Eq.~\eqref{eq:traj_con} and thus violates the in-distribution constraint.
In environments with probabilistic dynamics, such a conversion is not strictly equivalent, but the in-distribution constraint can still induce a trade-off between 
estimating some parts of the transition dynamics inaccurately (which would increase $D_{KL}(q(\tau)||p_{\pi_\beta}(\tau))$) and optimizing the expected return of the trajectory distribution (i.e., $ \mathbb{E}_{\tau\sim q(\tau)}[R(\tau)]$).
More theoretical analysis will be discussed in Subsection~\ref{sec:theoretical}. 



One of the advantages of optimizing the policy in the trajectory level is the capability of enabling strict constraint satisfaction, i.e., realizing the per-trajectory constraint with probability one in theory, which is in contrast to the constraint of the expectation over trajectories in most of previous studies.
Formally, for any given budget $b$, the per-trajectory safety constraint can be formulated as $\int_{\{\tau | C(\tau)\leq b\}}q(\tau)d\tau=1$
, where the integral is taken over all constraint-satisfied trajectories.
For brevity, we slightly abuse the notation $\int_{C(\tau)\leq b}$ to represent $\int_{\{\tau | C(\tau)\leq b\}}$ in the following.
Then, the offline constrained trajectory optimization problem can be given as follows:
\begin{align}
    &\max_{q(\tau)} \mathbb{E}_{\tau\sim q(\tau)}[R(\tau)]\label{eq:offlinerl_oj}\\
    \text{s.t. }&\int_{C(\tau)\leq b}q(\tau)d\tau=1 \label{eq:offlinerl_safe_cons}\\
    & D_{KL}(q(\tau)||p_{\pi_\beta}(\tau))\leq \epsilon.\label{eq:offlinerl_traj_cons}
\end{align}
While the 
closed-form solution satisfying both Eq.~\eqref{eq:offlinerl_safe_cons} and Eq.~\eqref{eq:offlinerl_traj_cons} does not necessarily exist, we show that the such a solution can be achieved provided a mild condition is met:

\begin{theorem}\label{thm:q^*_b}
    If the trajectory distribution of the behavior policy satisfies:
    \begin{equation}\label{eq:optimal_q_condition}
        \int_{C(\tau)\leq b}p_{\pi_\beta}(\tau)d\tau\geq e^{-\epsilon},
    \end{equation}
    then the optimal trajectory distribution for problem \eqref{eq:offlinerl_oj}\eqref{eq:offlinerl_safe_cons}\eqref{eq:offlinerl_traj_cons} exists and takes the following form: 
    \begin{equation}\label{eq:optimal_q}
            q^*_b(\tau)=
            \begin{cases} 
            p_{\pi_\beta}(\tau)\exp(\alpha R(\tau))/Z & \text{if }C(\tau)\leq b;\\
            0                           & \text{otherwise},
            \end{cases}
    \end{equation}
    where $\alpha$ is a constant depending on $\epsilon$ and $b$, and $Z=\int_{C(\tau)\leq b}p_{\pi_\beta}(\tau)\exp(\alpha R(\tau))d\tau$ is a constant normalizer to make sure that $q^*_b(\tau)$ is a valid distribution.
\end{theorem}
\begin{proof}
    Please refer to Appendix \ref{app:The Proof of Theorem 2}, Theorem~\ref{thm:q^*_b_2}.
\end{proof}
Note that condition \eqref{eq:optimal_q_condition} in Theorem~\ref{thm:q^*_b} states that the probability of the constraint-satisfied trajectories generated by the behavior policy should not be too small, which implies that there exist a sufficient number of constraint-satisfied trajectories in dataset. 
This condition can be easily satisfied for most of possible budgets since the behavior policy usually 
consists of multiple policies which generate the trajectories with diverse cost. 
Then, Theorem~\ref{thm:q^*_b} shows that the constrained optimal trajectory distribution $q^*_b$ can be simply obtained by zeroing the probability of unsafe trajectories and redistributing the probability of the remaining trajectories based on their returns.




\textbf{A Probabilistic Inference Perspective}
\quad The result similar to Eq.~\eqref{eq:optimal_q} can be also derived from the perspective of probabilistic inference~\cite{levine2018reinforcement}.
Specifically, an optimal variable $O$ is introduced to denote whether a given trajectory is optimal. 
Following the probabilistic graphical models in~\cite{levine2018reinforcement}, the likelihood of being optimal given trajectory $\tau$ can be represented as the exponential of the discounted cumulative reward: $P(O=1|\tau)\propto \exp(\alpha' R(\tau))$, where $\alpha'$ is a temperature parameter. 
Since the posterior we are concerned with (i.e., $p_{\pi_\beta}(\tau|O=1)$) is intractable, an auxiliary trajectory distribution $q(\tau)$ is introduced to serve as an approximation, and the discrepancy between $q(\tau)$ and $p_{\pi_\beta}(\tau|O=1)$ can be derived as (see Appendix~\ref{app:Probabilistic Inference Perspective} for the proof):
\begin{equation}
    \begin{aligned}\nonumber
        &D_{KL}(q(\tau)||p_{\pi_\beta}(\tau|O=1))\\
        =&\log P_{\pi_\beta}(O=1)-J(q(\tau)),
    \end{aligned}
\end{equation}
where $J(q(\tau))=-D_{KL}(q(\tau)||p_1(\tau))$
is the evidence lower bound (ELBO) and $p_1(\tau)\propto p_{\pi_\beta}(\tau)\exp(\alpha' R(\tau))$.
Note that $\log P_{\pi_\beta}(O=1)$ is independent of $q(\tau)$, hence the discrepancy between $q(\tau)$ and $p_{\pi_\beta}(\tau|O=1)$ can be minimized by optimizing the ELBO.

Constraining $q(\tau)$ within a feasible distribution family to satisfy the safety budget with probability one,
the optimization problem becomes the following form:
\begin{align}
        &\max_{q(\tau)\in \Pi_b}\ \ -D_{KL}(q(\tau)||p_1(\tau)) \label{eq:vinference_obj}\\
        &\Pi_b = \{q(\tau):\int_{C(\tau)\leq b}q(\tau)d\tau=1\}.\label{eq:vinference_con}
\end{align}
The solution of Eq.~\eqref{eq:vinference_obj}\eqref{eq:vinference_con} has the same form of  \eqref{eq:optimal_q} (see Appendix~\ref{app:Probabilistic Inference Perspective} for the detailed derivation), except that here $\alpha'$ (corresponding to $\alpha$ in Eq.~\eqref{eq:optimal_q}) is a tunable hyper-parameter and does not vary w.r.t.  budget $b$ and  divergence threshold $\epsilon$.
The explanation of the difference between $\alpha$ and $\alpha'$ is also provided in Appendix~\ref{app:Probabilistic Inference Perspective}.

\subsection{A Practical Algorithm}
\label{sec:traj_opti}
To obtain the optimal trajectory distribution in Theorem~\ref{thm:q^*_b}, we adopt Diffuser~\cite{janner2022planning}, a recently proposed trajectory optimization framework that has demonstrated significant scalability to long-horizon trajectories, as the backbone of our algorithm.
The Diffuser model first approximates the trajectory distribution in the dateset, and then performs optimization over the approximate distribution via a guided planning process.
Recall that this planning process requires the target optimal distribution to be in the form of $p_\theta(\tau)h(\tau)$. Thus we rewrite $q^*_b(\tau)$ in Eq.~\eqref{eq:optimal_q} as:
\begin{equation}
    \begin{aligned}
        q^*_b(\tau)=p_{\pi_\beta}(\tau)h_b(\tau),
    \end{aligned}\label{eq:p=ph}
\end{equation}
where $h_b(\tau)\propto \exp(\alpha\tilde R_b(\tau))$ and 
\begin{equation}
    \begin{aligned}
        \tilde R_b(\tau)=
        \begin{cases} 
            R(\tau)                     & \text{if }C(\tau)\leq b;\\
            -\infty                            & \text{otherwise}.
        \end{cases}
    \end{aligned}
\end{equation}
The form of distribution product in Eq.~\eqref{eq:p=ph} enables the Diffuser-like two-phase trajectory optimization, which consists of first approximating $p_{\pi_\beta}(\tau)$ and then performing guided planning over the approximate distribution via $h_b(\tau)$.
Note that an infinite $\tilde R_b(\tau)$ would lead to value mutations of $h_b(\tau)$ and may thus do harm to the guiding process of the diffusion model, as mentioned in Subsection~\ref{sec:diffuser}.
Therefore, we introduce an alternative trajectory distribution to $q_b^*(\tau)$:
\begin{equation}
    \begin{aligned}
        q^*_{b,n}(\tau)&=p_{\pi_\beta}(\tau)h_{b,n}(\tau),
    \end{aligned}\label{eq:pn=phn}
\end{equation}
where $h_{b,n}(\tau)\propto\exp(\alpha\tilde R_{b,n}(\tau))$, and
\begin{equation}
    \begin{aligned}\label{eq:approx_p=pn}
        \tilde R_{b,n}(\tau)=R(\tau)-n\cdot \mathbb{I}(C(\tau)>b)\cdot(C(\tau)-b).
    \end{aligned}
\end{equation}
Intuitively, the hyper-parameter $n$ controls the ``conservativeness'' of the planning process,
meaning that a larger $n$ induces a smaller probability of constraint violation.
Formally,
\begin{equation}
    \begin{aligned}
        \lim_{n\rightarrow \infty}\int_{C(\tau)> b}q^*_{b,n}(\tau)d\tau = 0. 
    \end{aligned}
\end{equation}
%
\begin{algorithm}[tb]
    \caption{\textbf{T}rajectory-based \textbf{RE}al-time \textbf{B}udget \textbf{I}nference}
    \label{alg:algorithm1}
    \textbf{Require}: Diffuser model $\mu_\theta$, parameter $\alpha$, covariances~$\Sigma^i$, budget $b$,  hyper-parameter $n$; 
    \begin{algorithmic}[1] 
        \STATE $z_0=b$;
        \FOR{$t = 0,...,T$}
        \STATE Observe state s; initialize trajectories $\tau^N\sim \mathcal{N}(0,\mathcal{I})$;
        \FOR{$i = N,...,1$}
        \IF {$C(\tau^i)\leq z_t$}
        \STATE $g=\alpha\nabla R(\tau^i)$;
        \ELSE
        \STATE $g=\alpha(\nabla R(\tau^i)-n\cdot\nabla C(\tau^i))$;
        \ENDIF
        \STATE $\mu \leftarrow \mu_\theta(\tau^i)$;
        \STATE $\tau^{i-1}\sim \mathcal{N}(\mu+ \Sigma^i g, \Sigma^i)$;
        \STATE $\tau^{i-1}\leftarrow s$;
        \ENDFOR
        \STATE Execute the first action $\tau_{a_0}$ and get cost~$c_t$; 
        \STATE $z_{t+1}=(z_t-c_t)/\gamma$.
        \ENDFOR
      \end{algorithmic}
\end{algorithm}

To sample from the approximate optimal distribution $q^*_{b,n}$, we adapt the guided planning of Diffuser (i.e., the modified reverse process~\eqref{eq:modified_reverse}) to the constrained setting by replacing $h(\tau)$ and $g$ respectively with $h_{b,n}(\tau)$ and $g_b$, where
\begin{align}
        g_b&=\nabla_{\tau^i} \log h_{b,n}({\tau^i})\nonumber\\
        &=
        \begin{cases} 
            \alpha \nabla_{\tau^i} R({\tau^i})                    & \text{if }C({\tau^i})\leq b;\\
            \alpha(\nabla_{\tau^i} R({\tau^i})-n\nabla_{\tau^i} C({\tau^i}))                            & \text{otherwise}.
        \end{cases}\label{eq:gb}
\end{align}
%
    
By assigning different budget values to $b$ in Eq.~\eqref{eq:gb} during inference, the trajectories satisfying the corresponding budget constraint can be sampled from the trained model without time-consuming retraining, which is the major advantage of our approach in addressing the real-time budget constraint problem.
In addition, in practice the reward and the cost value function (i.e., $R$ and $C$ in Eq.~\eqref{eq:gb}) are often unknown and need to be approximated.
To reduce the long-horizon approximation error, the trajectory optimization are re-performed at some predetermined time intervals~\footnote{The ablation study of the time interval is shown in Appendix~\ref{app:control_freq}.} during a long-horizon environment episode, based on the current state and the remaining budget.
Specifically, at the $t^\text{th}$ time step, the remaining trajectory $\tau$ is regenerated by taking $s_t$ as the  first state $\tau_{s_0}$ and $z_t=(b-\sum_{k=0}^{t-1}\gamma^k c(s_k,a_k))/\gamma^t$ as a scaled version of the remaining safety budget.
Formally, the following time-varying constraint should be satisfied during each trajectory optimization:
\begin{equation}
    \begin{aligned}
    C(\tau)&\leq z_t\\
    \text{where\qquad}\tau_{s_0}&=s_t,\nonumber\\
    z_0&=b,\nonumber\\
    z_t=(z_{t-1}&-c(s_{t-1},a_{t-1 }))/\gamma.\nonumber
    \end{aligned}
\end{equation}
Keeping aware of the remaining budget, TREBI can adjust its ``conservativeness'' flexibly with the time steps to better ensure constraint satisfaction.
The pseudo code of TREBI is presented in Algorithm \ref{alg:algorithm1}.

\subsection{Theoretical Analysis}
\label{sec:theoretical}
In this subsection, we derive a performance guarantee for the reward maximization and constraint satisfaction of our optimal trajectory distribution in Proposition~\ref{thm:q^*_b}, and provide analysis and interpretations of the theoretical results.
To begin with, we present an error bound of the offline estimation on the episodic reward and cost, which serves as a useful lemma in the derivation of our conclusion:

\begin{lemma}\label{lemma:bound}
    Let $\pi$ be a policy derived from $q(\tau)$.
    Denote $J_{T_q}(\pi)=\mathbb E_{\tau\sim q(\tau)}[R(\tau)]$ as the expected return of $q(\tau)$ and $J_T(\pi)$ as the expected return of $\pi$ under the real environment $T$. Let $C(\tau)$ and $\hat C(\tau)$ be the episodic cost of trajectory $\tau$ under the $T$ and empirical dynamics $\hat T$, respectively.
    With probability at least $1-\delta$, the gap between $J_{T_q}(\pi)$ and $J_T(\pi)$ is bounded by:
    \begin{equation}
        \begin{aligned}\label{eq:bound_return}
            |J_{T_q}(\pi)- J_{T}(\pi)|\leq & \frac{2R_{m}}{1-\gamma}\bigg[(1-\gamma^{L+1})\sqrt{\frac{L\epsilon}{2}}+\\
            &\mathbb E_{\pi, \hat T}\left[\sum_{t=0}^L\frac{\gamma^t C_{T,\delta}}{\sqrt{{N}(s_t,a_t)}}\right]\bigg],
        \end{aligned}
    \end{equation}
    where the first term $(1-\gamma^{L+1})\sqrt{\frac{L\epsilon}{2}}$ vanishes in deterministic environment,
    and for each feasible trajectory $\tau$,
    \vskip -0.2in
    \begin{equation}
        \begin{aligned}\label{eq:bound_cost}
            |C(\tau)-\hat C(\tau)|\leq \sum_{t=0}^L\frac{\gamma^t C_{c,\delta}}{\sqrt{{N}(\tau_{s_t},\tau_{a_t})}},
        \end{aligned}
    \end{equation}
   where $C_{T,\delta},C_{c,\delta}$ are constants depending on the concentration properties of $T(s'|s,a),c(s,a)$, respectively, and $\delta\in(0,1)$. $N(s,a)$ is the counts for each state-action pair $(s,a)$ in the offline dataset.
\end{lemma}
\begin{proof}
    Please refer to Appendix~\ref{app:proof2}, Lemma~\ref{lemma:bound_2}. 
\end{proof}
According to Eq.~\eqref{eq:bound_return}, the estimation error of the expected return is controlled mainly by two terms: $(1-\gamma^{L+1})\sqrt{\frac{L\epsilon}{2}}$ and $\mathbb E_{\pi, \hat T}\left[\sum_{t=0}^L\frac{\gamma^t C_{T,\delta}}{\sqrt{{N}(s_t,a_t)}}\right]$.
The first term is caused by the inconsistency between $\hat{T}$ and the dynamics induced by the optimized $q(\tau)$.
As described in Subsection~\ref{sec:rl_to_match}, this inconsistency only exists in environment with probabilistic dynamics.
Lemma~\ref{lemma:bound} theoretically show that such an inconsistency in probabilistic dynamics can be controlled by $\epsilon$.
Thus, applying a strict in-distribution constraint can reduce the dynamics inconsistency in the probabilistic case, but at a cost of limiting the parameter search space for trajectory optimization.
The second term 
represents the uncertainty of dynamics on the trajectory generated by $\pi$ and $\hat T$.
This uncertainty generally decreases as the size of dataset increases. 
Besides, this term can be further replaced by $\frac{2(1-\gamma^{L+1})}{1-\gamma}\sqrt{\frac{L\epsilon}{2}}+\mathbb E_{\pi_\beta, \hat T}\left[\sum_{t=0}^L\frac{\gamma^t C_{T,\delta}}{\sqrt{{N}(s_t,a_t)}}\right]$ (see Appendix~\ref{app:proof2} for the proof), where the two terms respectively represent the proximity of $\pi$ to behavior policy $\pi_\beta$ and the uncertainty of $\pi_\beta$.
Thus, by adopting a stricter in-distribution constraint, the error caused by dynamics uncertainty can be reduced toward the uncertainty of behavior policy, which can be regarded approximately as a lower bound of the feasible policy's uncertainty.


As for the trajectory cost estimation, the bound in Eq.~\eqref{eq:bound_cost} shows that the more in-distribution state-action pairs the trajectory $\tau$ contains, the more accurate its cumulative cost estimation will be. Therefore, employing the constraint in Eq.~\eqref{eq:traj_con} induces better constraint satisfaction.
Having Lemma~\ref{lemma:bound}, the performance guarantee for the optimal trajectory distribution $q^*_{b}(\tau)$ can be easily derived as:

\begin{proposition}
\label{thm:optimal_guarantee}
    Denoting $\pi_{q^*_b}$ as the policy induced by the constrained optimal trajectory distribution $q^*_b$,
    for any policy $\pi$ derived from the trajectory distribution that simultaneously satisfies the in-distribution constraint~\eqref{eq:offlinerl_traj_cons} and the safety constraint~\eqref{eq:offlinerl_safe_cons}, the following inequality holds with probability at least $1-\delta$:
    \begin{equation}
        \begin{aligned}\label{eq:return_bound}
            J_T(\pi_{q^*_b})\geq  J_T(\pi)-&\frac{4R_{m}}{1-\gamma} \bigg[(1-\gamma^{L+1})\sqrt{\frac{L\epsilon}{2}}\\
            &+\mathbb E_{\pi, \hat T}\left[\sum_{t=0}^L\frac{\gamma^t C_{T,\delta}}{\sqrt{{N}(s_t,a_t)}}\right]\bigg],
        \end{aligned}
    \end{equation}
    where the first term $(1-\gamma^{L+1})\sqrt{\frac{L\epsilon}{2}}$ vanishes in  environments with deterministic dynamics.
   For each trajectory $\tau$ generated from $q^*_b$, the episodic cost of $\tau$ is bounded with probability at least $1-\delta$ as follow:
    \begin{equation}
        \begin{aligned}\label{eq:cost_bound}
            C(\tau)\leq b+\sum_{t=0}^L\frac{\gamma^t C_{c,\delta}}{\sqrt{{N}(\tau_{s_t},\tau_{a_t})}}.
        \end{aligned}
    \end{equation}
\end{proposition}
\begin{proof}
    Please refer to Appendix \ref{app:proof2}, Proposition~\ref{thm:optimal_guarantee_2}.
\end{proof}
\vskip -0.1in
In the case of deterministic dynamics, $\pi_{q_b}^*$ can be almost as good as any policy derived from the trajectory distribution that satisfies both the in-distribution constraint and the safety constraint, provided that uncertainty of dataset is small, which generally corresponds to having sufficient samples in the datase.
In the case of probabilistic dynamics, the performance gap between $\pi_{q_b}^*$ and other feasible policies can be controlled by the in-distribution constraint.
Eq.~\eqref{eq:cost_bound} suggests that when the trajectories generated from $q^*_b(\tau)$ are in-distribution, the upper bound of the cumulative cost is more likely to be controlled.



\section{Experiments}
\label{sec:experiments}
\subsection{Simulation Experiments}
\begin{figure}[t]
\begin{center}
\centerline{\includegraphics[width=\columnwidth]{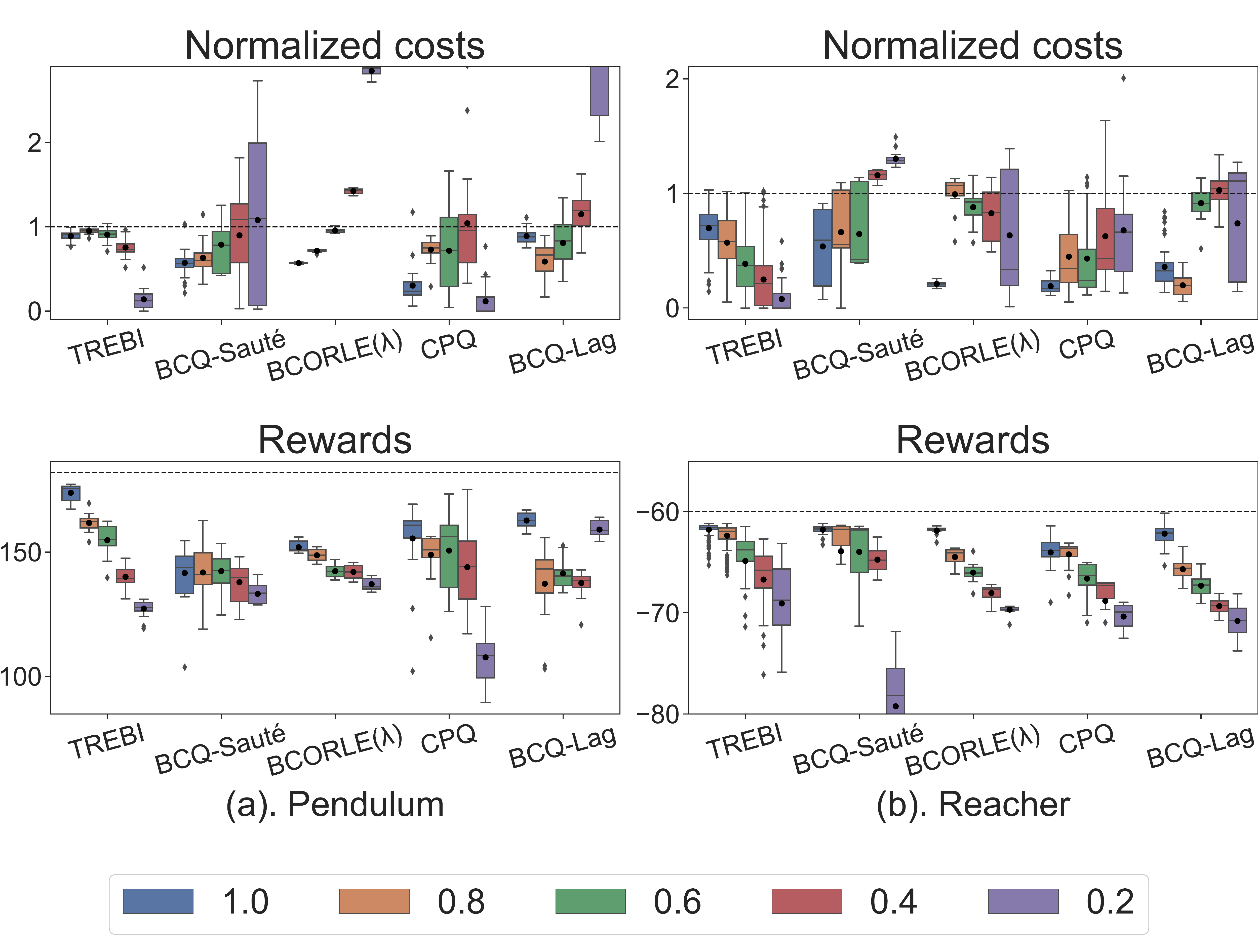}}
\caption{The results on Pendulum and Reacher with regard to the normalized episodic costs (the top row) and the episodic rewards (the bottom row), under the constraints w.r.t. five different budgets (bars with different colors).
The constraint w.r.t. a budget is satisfied only when the value of the normalized cost is less than one.
The dashed lines in the cost plot and the reward plot  indicate the normalized episodic cost threshold and the expected episodic reward of the converged unconstrained policy trained online, respectively.
Note that for the convenience of display, some of the boxes are not completely shown in the figure, e.g., the reward of BCQ-Sauté on Reacher with budget $0.2$.
}
\label{fig:gym env}
\end{center}
\vskip -0.28in
\end{figure}

\label{sec:sim_exp}
\textbf{Environments and Datasets}
\quad We first validate TREBI's capability in solving real-time budget constraint problems on two OpenAI Gym tasks with additional safety constraints (Pendulum swing-up and Reacher)~\cite{sootla2022saute}, three MuJoCo tasks (Hopper-v2, HalfCheetah-v2, Walker2d-v2)~\cite{todorov2012mujoco} with speed limit~\cite{zhang2020first,yang2022cup} and two Bullet-Safety-Gym tasks (SafetyCarCircle-v0, SafetyBallReach-v0)~\cite{gronauer2022bullet}. 
For Pendulum swing-up, Reacher and two Bullet-Safety-Gym tasks, we train an unconstrained SAC policy online and take its replay buffer as the offline dataset.
For MuJoCo tasks, we test on three types of datasets in the D4RL benchmark~\cite{fu2020d4rl}, i.e., medium, medium-replay and medium-expert, to verify TREBI's capability of utilizing datasets with various levels of quality. 
Note that the offline data in our experiments are collected by \emph{unconstrained} behavior policies. More environmental introduction and implementation details can be found in Appendix~\ref{app:environment_detail}.

\begin{figure*}[ht!]
\begin{center}
\centerline{\includegraphics[width=\textwidth]{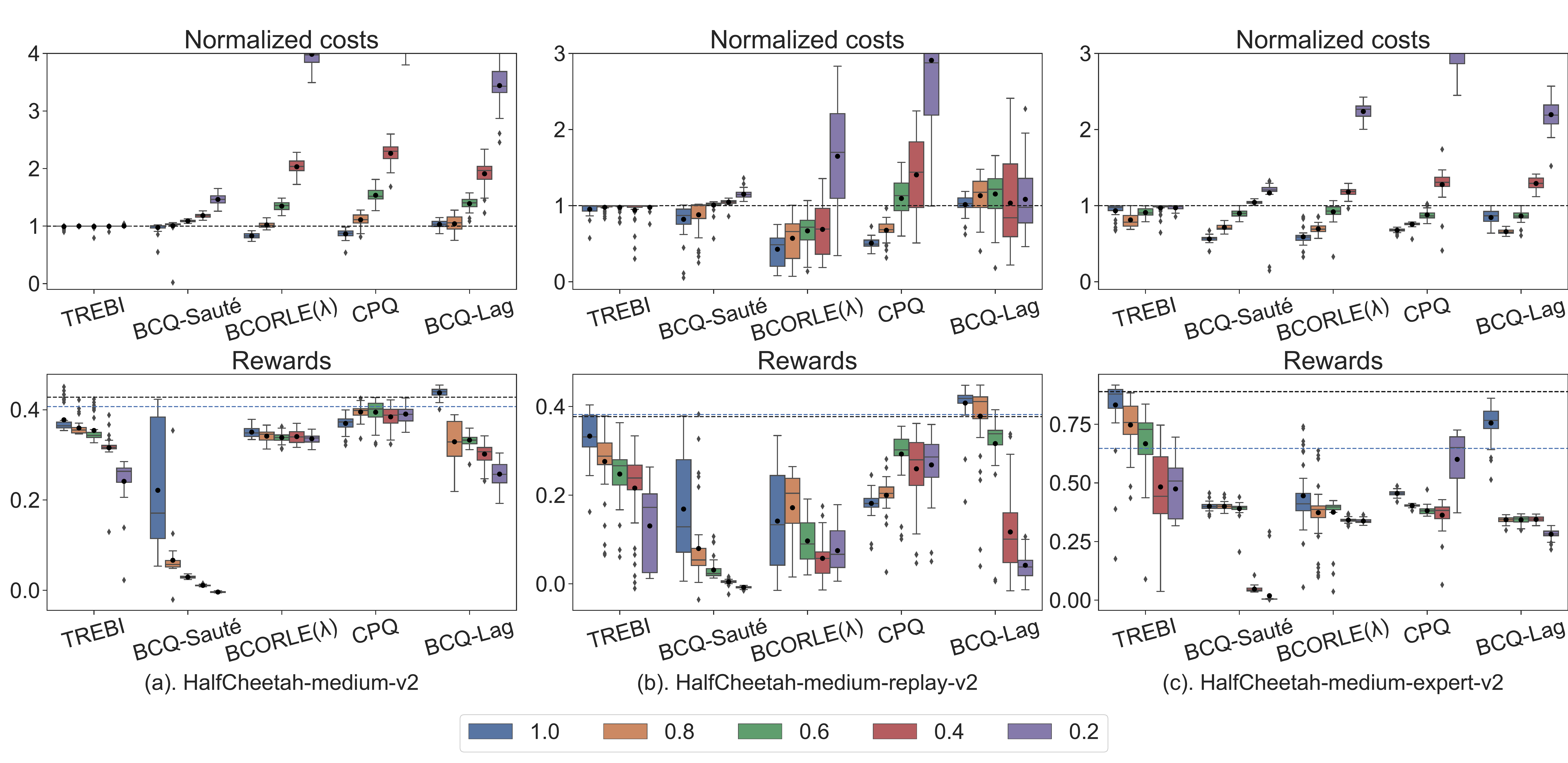}}
\caption{Results on HalfCheetah using three different types of datasets. 
The black dashed lines and the blue dashed lines in the reward plots indicate the expected episodic rewards of the unconstrained Diffuser and BCQ respectively.
} 
\label{fig:halfcheetah}
\end{center}
\vskip -0.25in
\end{figure*}

\begin{figure}[ht!]
\begin{center}
\centerline{\includegraphics[width=\columnwidth]{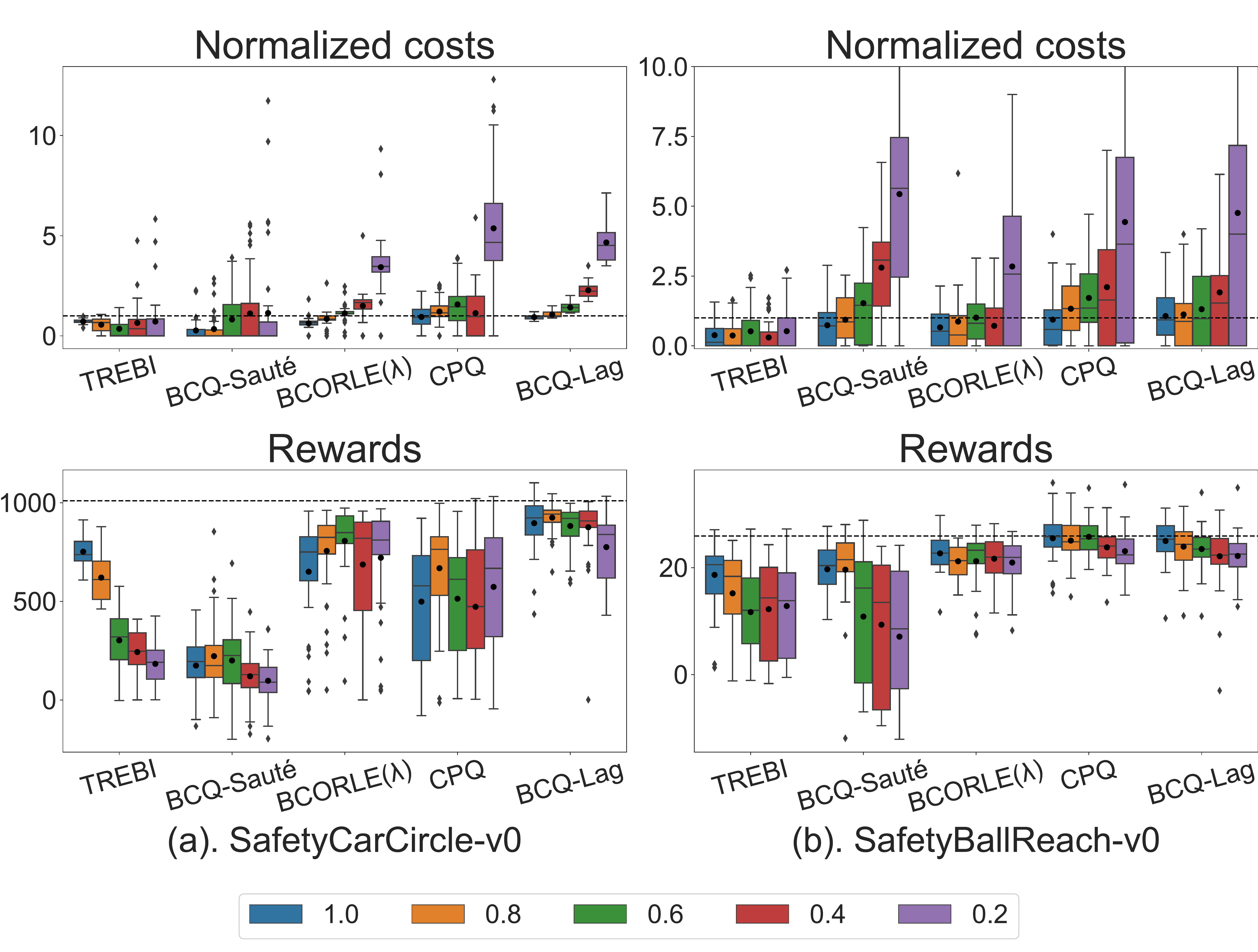}}
\caption{The results on SafetyCarCircle and SafetyBallReach.
The dashed lines in reward plots indicate the expected episodic reward of the behavior policy in the final iteration of online training.
}
\label{fig:bulletsafetygym}
\end{center}
\vskip -0.28in
\end{figure}

\textbf{Baselines}
$\quad$We compare TREBI~\footnote{Implementation details can be found in Appendix~\ref{app:trebi_imple}.} with: 1) BCQ-Lagrangian, which combines BCQ~\cite{fujimoto2019off} with the Lagrangian approach; 
2) CPQ~\cite{xu2022constraints}, which considers out-of-distribution actions as unsafe and updates the reward critic using only safe state-action pairs; 
3) BCORLE($\lambda$)~\cite{zhang2021bcorle}, which merges Lagrangian multiplier $\lambda$ into the state and finds an optimal value of $\lambda$ for a given budget by bisection search; 
4) BCQ-Sauté, which combines BCQ with Sauté~\cite{sootla2022saute} by augmenting the state space with the remaining safety budget and assigning large penalty values to the reward when the remaining safety budget is less than 0.
Note that BCORLE($\lambda$) and BCQ-Sauté have the potential to handle different budgets, as the most of possible budgets have been taken into account during policy training by state augmentation. 
\textbf{Evaluation Protocols}
$\quad$The approaches are evaluated on real-time budgets varying across 5 different values, i.e., $0.2,0.4,0.6,0.8$ and $1.0$ times of the maximum budget respectively\footnote{In Appendix~\ref{app:uniform_sampled_budget}, we present the results of additional experiments where the ratios to the maximum budget are sampled from interval $[0,1]$ rather than 5 fixed values.}. 
The maximum budgets for each tasks are detailed in Appendix~\ref{app:hyper-parameters}.
We normalize the episodic cost of a trajectory by dividing it with the corresponding budget, resulting in a unified constraint budget of one.
We adopt the box-and-whisker plots to show the results of constraint satisfaction in the trajectory level. 
Concretely, the box shows the median, $q_3$ and $q_1$ quartiles of the distributions, and whiskers depict the error bounds $1.5(q_3-q_1)$ and outliers, while the mean is marked by a black dot.
Such indicators serve as crucial criteria for the validation of constraint satisfaction with probability one.
Note that among the baselines, BCQ-Lagrangian and CPQ are only applicable to fix-budget settings by design, while BCORLE($\lambda$) and BCQ-Sauté can handle different safety budgets without retraining.
Therefore, BCQ-Lagrangian and CPQ are retrained for each budget before they are evaluated, while BCORLE($\lambda$), BCQ-Sauté and TREBI are trained only once before evaluation for each budget.


\textbf{Comparative Evaluations}
$\quad$ Figures~\ref{fig:gym env}, \ref{fig:halfcheetah} and \ref{fig:bulletsafetygym} show the results for Pendulum, Reacher, HalfCheetah and two BulletSafetyGym tasks.
Due to space limit, the results on other MuJoCo tasks are placed in Appendix~\ref{app:full_mujoco_results}.
Overall, TREBI incurs less constraint violation compared to the baselines and achieves the competitive performance on the episodic reward.
Importantly, TREBI achieves a high level of per-trajectory constraint satisfaction on most of the tasks and budgets, which is consistent with our theoretical setting. 
In contrast, the baselines not only show large variances in the episodic costs, but also suffer from a high probability of serious constraint violation especially when the budget is low.
Moreover, the consistent superior performance across the three different types of MuJoCo datasets also serves as a proof for TREBI's capability in utilizing data with various levels of quality.

\begin{figure}[t]
\begin{center}
\centerline{\includegraphics[width=\columnwidth]{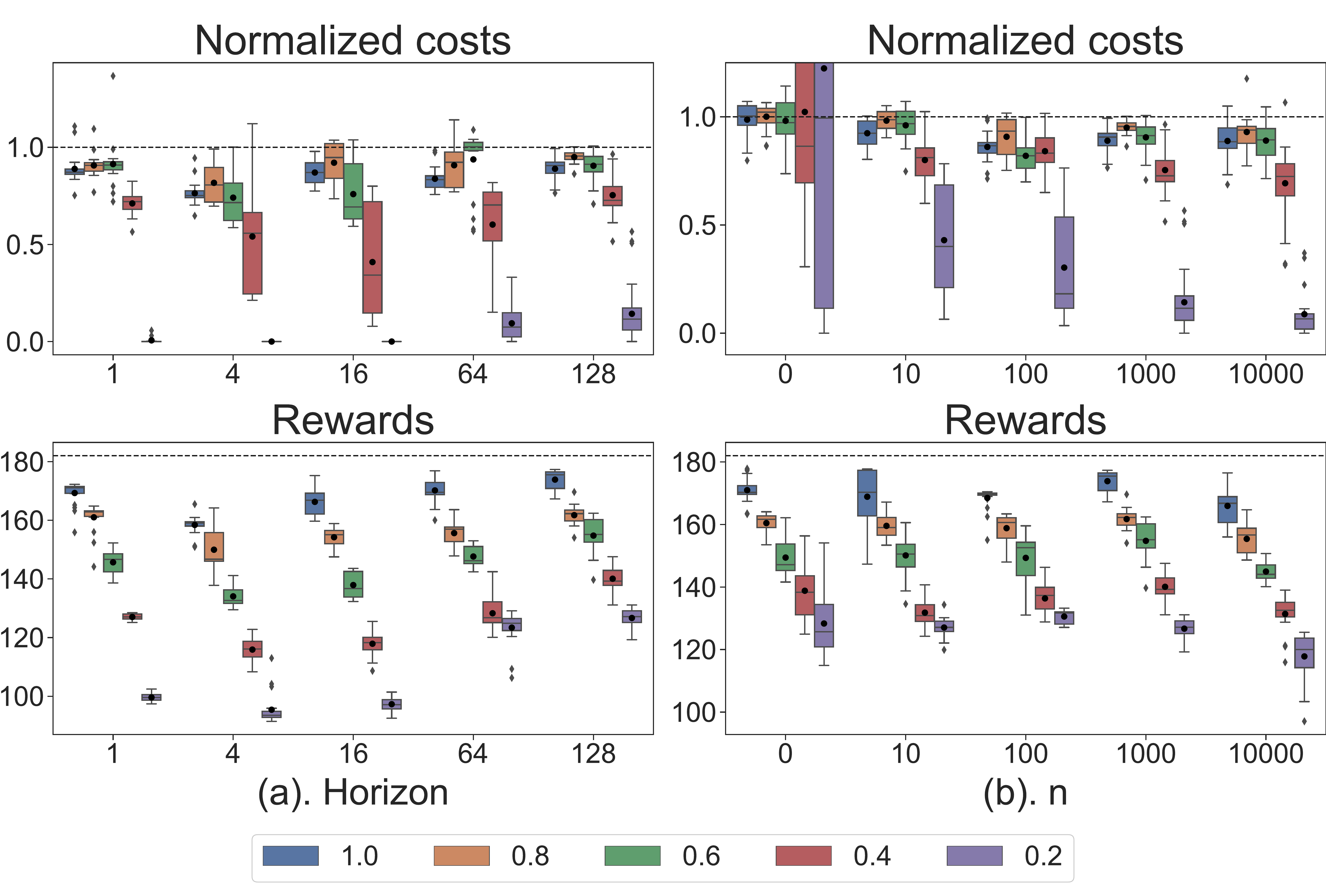}}
\caption{
Ablation studies on Pendulum. 
} 
\label{fig:ablation}
\end{center}
\vskip -0.4in
\end{figure}

\textbf{Ablation 1}
\quad It is often impractical to sample full-length episodes in planning due to the computational complexity. Hence, following the implementation of Diffuser, we sample and optimize the trajectories within only a fixed horizon. 
The effect induced by different horizon settings are given in Figure~\hyperref[fig:ablation]{\ref*{fig:ablation}(a) }, showing a positive trend (higher reward \& better constraint satisfaction) as the horizon increases.

\textbf{Ablation 2}
\quad
In Eq.~\eqref{eq:approx_p=pn}, the parameter  $n$ is introduced to avoid numerical issues of the guiding process.
Figure~\hyperref[fig:ablation]{\ref*{fig:ablation}(b)} provides the ablation result of $n$.
Roughly speaking, when $n$ varies within a relatively low-value range, a larger $n$ leads to better constraint satisfaction and higher rewards.
However, when $n$ comes to an extremely large value, e.g., $10000$, the above two metrics instead become worse, which verifies that a large $n$ can degrade the smoothness of the guiding distribution $h(\tau)$ and thus has a detrimental effect on the modified reverse process~\cite{sohl2015deep}.

\subsection{Real-world Environments}
\label{sec:real_world_exp}


We further evaluate TREBI in a real-world advertising bidding scenario w.r.t. the optimized cost-per-mille (oCPM).

\textbf{Environment}
$\quad$In the advertising bidding scenario, a sequence of advertisements (ads) slots is going to be presented to users browsing the site, and the advertisers would bid for these slots to display their advertisements.
Advertisers assign the target ROI (return on investment) and their maximum total payments for the advertisements.
Then, the policy trained by the auto-bidding algorithm helps the advertisers bid for each slot.
The advertisers consume their budgets and obtain the revenue only when their ads are successfully displayed, and the policy needs to ensure the average ROI to be higher than the assigned target.
Thus, the objective is maximizing the summation of the payments of all advertisers while satisfying the ROI constraint.
In this scenario, a major challenge is to respond to the ROI targets set by each advertiser, which would change once in a day.
Due to the high risk of online training, obtaining a safe policy using merely offline data is promising.
Appendix~\ref{app:bid} provides more detailed information of this scenario.

\begin{table}[t]
\caption{Results in the oCPM advertising bidding task. 
}
\label{tab:bid}
\vskip 0.1in
\begin{center}
\begin{small}
\begin{tabular}{lccr}
\toprule
Algorithm &   Relative Improvement &  Violation rate \\
\midrule
\textbf{TREBI}    & \textbf{7.15\%} & \textbf{14.69\%}   \\
BCORLE($\lambda$)   & 1.55\% & 18.66\%   \\
BCQ-Sauté  & -2.55\% & 21.49\%   \\
CPQ     & -0.76\% & 20.75\%   \\
BCQ-Lag   & 0.00\% & 22.94\%   \\
\bottomrule
\end{tabular}
\end{small}
\end{center}
\vskip -0.15in
\end{table}

\textbf{Dataset and Baselines}
\quad The offline dataset is collected from a real-world ads bidding platform, including 1938151000 bidding records of 42376 advertisers sampled from a one-month bidding history.
The baselines adopted in this experiment are the same as those in Subsection~\ref{sec:real_world_exp}.

\textbf{Evaluation Protocol}
\quad We deploy the algorithms on the real-world ad bidding platform for 10 days, conducting evaluations on the 485904618 bidding records randomly sampled from 29660 advertisers.
As explained in Subsection~\ref{sec:sim_exp}, BCORLE($\lambda$) and BCQ-Sauté are trained only once by augmenting the samples w.r.t. the real-time budgets.
As for BCQ-Lagrangian and CPQ, since the budget varies from 20 to around 2000 while the online deployment requires real-time responses to the budgets, it is impractical to retrain these two baselines on each budget.
Thus, we train BCQ-Lagrangian and CPQ only once using the mean of budget in the dataset as the fixed budget.

\textbf{Comparative Evaluation}
\quad Table~\ref{tab:bid} reports the total pay improvement relative to BCQ-Lagrangian and the violation rate, which is the proportion of trajectories violating ROI constraints during evaluation.
Due to the extremely large scale of the bidding platform and the dramatic variation of the real-time budgets, none of the algorithms can achieve the $100\%$ constraint satisfaction. 
However, compared to the baselines, TREBI achieves a higher episodic reward and generates fewer trajectories violating the constraint, which demonstrates the potential of TREBI to solve the real-time budget constraint problem in real-world applications. 

\section{Conclusion}
In this paper, we proposed TREBI, a novel solution to the real-time budget constraint problem under the offline setting.
We view the safe offline RL problem from a trajectory optimization perspective and solve it by a planning method, which enables the strict constraint satisfaction in the trajectory level as well as efficient utilization of unconstrained offline dataset. We provide an analysis on the performance guarantee for TREBI theoretically, and verify its capability empirically in addressing the real-time budget constraint problems in a wide range of simulation and real-world tasks.  

Inheriting from the Diffuser method, our method would be more computationally expensive during inference than non-planning algorithms, due to the slow iterative trajectory generation of the diffusion model.
In this regard, several improvements can be made to reduce the time complexity, such as warm-starting the diffusion process~\cite{janner2022planning} and reducing the control frequency (referred to Appendix~\ref{app:control_freq}). 
In the future, we plan to study the technology for accelerating diffusion planning and explore other trajectory optimization frameworks to further improve the efficiency of TREBI.

\section*{Acknowledgements}
We gratefully acknowledge support from the National Natural Science Foundation of China (No. 62076259), the Fundamental and Applicational Research Funds of Guangdong province (No. 2023A1515012946), and the Fundamental Research Funds for the Central Universities-Sun Yat-sen University.
This research was supported by Meituan. 


\bibliography{example_paper}
\bibliographystyle{icml2023}

\newpage
\appendix
\onecolumn
\section{Proofs and Derivations}
\setcounter{theorem}{0}

\subsection{Proof of Theorem~\ref{thm:q^*_b}}
\label{app:The Proof of Theorem 2}
\begin{theorem}[4.1 in the main paper]
\label{thm:q^*_b_2}
    If $\int_{C(\tau)\leq b}p_{\pi_\beta}(\tau)d\tau\geq e^{-\epsilon}$ holds, there exists an optimal trajectory distribution for problem 
    \begin{align}
        &\max_{q(\tau)} \mathbb{E}_{\tau\sim q(\tau)}[R(\tau)]\label{eq:offlinerl_oj_2}\\
        \text{s.t. }&\int_{C(\tau)\leq b}q(\tau)d\tau=1 \label{eq:offlinerl_safe_cons_2}\\
        & D_{KL}(q(\tau)||p_{\pi_\beta}(\tau))\leq \epsilon,\label{eq:offlinerl_traj_cons_2}
    \end{align}
    which has the form of 
    \begin{equation}\label{eq:q^*_b}
            q^*_b(\tau)=
            \begin{cases} 
            p_{\pi_\beta}(\tau)\exp(\alpha R(\tau))/Z & \text{if }C(\tau)\leq b;\\
            0                           & \text{otherwise},
            \end{cases}
    \end{equation}
    where $\alpha$ depends on $\epsilon$ and $b$, and $Z=\int_{C(\tau)\leq b}p_{\pi_\beta}(\tau)\exp(\alpha R(\tau))d\tau$ is a constant normalizer to make sure that $q^*_b(\tau)$ is a valid distribution.
\end{theorem}

Before presenting the proof of Theorem~\ref{thm:q^*_b}, we first prove a lemma that will be used later:

\begin{lemma}\label{lemma:lemma1}
    Assuming that $x$ is a continuous variable, $p(x)$ and $q(x)$ are functions w.r.t. $x$, $\mathcal A \subseteq \mathcal{X} $ where $\mathcal{X}$ is the domain of $x$ by definition, the optimization problem
    \begin{align}
        \min_{q}\; & J(q) = \int_{\mathcal{A}}q(x)\log\frac{q(x)}{p(x)},\label{eq:lemma1_ob}
        \\
    \text{s.t. }& \int_{\mathcal{A}}q(x) = 1\label{eq:lemma1_cons}.
    \end{align}
    have the solution: $q^{*}(x) = \frac{p(x)}{\int_{\mathcal{A}}p(x)dx}$.
\end{lemma}
\begin{proof}
We first construct the Lagrange function as:
\begin{equation}
    \begin{aligned}
        L(x,q,\dot{q}) = \int_{\mathcal{A}}q(x)\log\frac{q(x)}{p(x)} + \lambda q(x).
        \nonumber
    \end{aligned}
\end{equation}
It is obvious that the original problem has a minimum point $q^{*}$. Thus, there exists a constant multiplier $\lambda$ such that $q^{*}$ satisfies the Euler-Lagrange Equation:
\begin{equation}
    \begin{aligned}
        \frac{\partial{L}}{\partial q} - \frac{d}{dx}\frac{\partial{L}}{\partial{ \dot q}} = 0,
        \nonumber
    \end{aligned}
\end{equation}
which yields $\frac{\partial{L}}{\partial q}=0$ and thus
\begin{equation}
    \begin{aligned}
        \frac{\partial{L}}{\partial q}=\log\frac{q(x)}{p(x)}+1+\lambda=0
\\
\Longrightarrow\;
q(x) = p(x)\cdot \exp(-(1+\lambda)).
        \nonumber
    \end{aligned}
\end{equation}
Since $\int_{\mathcal{A}}q(x)dx=1$, we can easily get $\exp(1+\lambda) = \int_{\mathcal{A}}p(x)dx$. 
As a result, the minimum point can be expressed by $q^{*}(x) = \frac{p(x)}{\int_{\mathcal{A}}p(x)dx}$.
In the case when x is a discrete variable,  we can easily get a similar solution, i.e., $p^{*}(x) = \frac{p(x)}{\sum_{x\in A} q(x)dx}$.
\end{proof}


Now we are ready to prove Theorem~\ref{thm:q^*_b}:

\begin{proof}
We first assume 
that $\int_{C(\tau)\leq b}q(\tau)d\tau$ is arbitrarily closed to 1 rather than directly equals 1 in our derivation, by which we avoid the occurrence of $p(\tau)=0$ in the calculation of the KL divergence.
We denote $\int_{C(\tau)\leq b}q(\tau)d\tau$ as $P_c$. 
When $P_c\rightarrow 1$, the probability of constrain-violating trajectory is tend to 0. Thus, 
\begin{equation}
    \begin{aligned}
        \lim_{P_c\rightarrow 1}\int_{C(\tau)> b}q(\tau)\log \frac{q(\tau)}{p_{\pi_\beta}(\tau)} d\tau=\int_{C(\tau)> b}\lim_{q(\tau)\rightarrow 0}q(\tau)\log q(\tau)-q(\tau) \log p_{\pi_\beta}(\tau) = 0.
        \nonumber
    \end{aligned}
\end{equation}

When $P_c \rightarrow 1$, the divergence of distribution is
\begin{equation}
    \begin{aligned}
    &\lim_{P_c\rightarrow 1} D_{KL}(q(\tau)||p_{\pi_\beta}(\tau))\\
    =&\lim_{P_c\rightarrow 1}\left(\int_{C(\tau)> b}q(\tau)\log \frac{q(\tau)}{p_{\pi_\beta}(\tau)} d\tau+\int_{C(\tau)\leq b}q(\tau)\log \frac{q(\tau)}{p_{\pi_\beta}(\tau)} d\tau\right)\\
    =&\int_{C(\tau)\leq b}q(\tau)\log \frac{q(\tau)}{p_{\pi_\beta}(\tau)} d\tau.
    \nonumber
    \end{aligned}
\end{equation}
Then, we study the existence condition of the solution to problem \eqref{eq:offlinerl_oj_2}\eqref{eq:offlinerl_safe_cons_2}\eqref{eq:offlinerl_traj_cons_2}.
We first derive the trajectory distribution closest to the behavior distribution only under the safety constraint Eq.~\eqref{eq:offlinerl_safe_cons_2}, i.e.,
\begin{equation}
    \begin{aligned}
    \min_{q(\tau)}\ &D_{KL}(q(\tau)||p_{\pi_\beta}(\tau))\\
    =&\int_{C(\tau)\leq b}q(\tau)\log \frac{q(\tau)}{p_{\pi_\beta}(\tau)} d\tau\\
    \text{s.t.} & \int_{C(\tau)\leq b}q(\tau)d\tau=1
    \nonumber
    \end{aligned}
\end{equation}
Replacing $p(x),q(x),\mathcal A$ with $p_{\pi_\beta}(\tau),q(\tau),\{\tau|C(\tau)\leq b\}$, respectively, we apply Lemma~\ref{lemma:lemma1} to get the optimal distribution $q^{*}(\tau) = \frac{p_{\pi_\beta}}{\int_{C(\tau)\leq b} p_{\pi_\beta}d\tau}$, which is the closest distribution to $p_{\pi_\beta}(\tau)$ under the safety constraint. 
There exists a distribution $q(\tau)$  satisfying both the safety constraint~\eqref{eq:offlinerl_safe_cons_2} and the in-distribution constraint~\eqref{eq:offlinerl_traj_cons_2} only when the divergence between $p_{\pi_\beta}(\tau)$ and the closest safe distribution $q^{*}(\tau)$ is less than $\epsilon$:
\begin{equation}
    \begin{aligned}
    D_{KL}(q^{*}(\tau)||p_{\pi_\beta}(\tau))&=\int_{C(\tau)\leq b}q(\tau)\log \frac{p_{\pi_\beta}(\tau)}{p_{\pi_\beta}(\tau)\int_{C(\tau)\leq b} p_{\pi_\beta}d\tau} d\tau\\
    &=-\log \int_{C(\tau)\leq b} p_{\pi_\beta}d\tau \leq \epsilon,
    \nonumber
    \end{aligned}
\end{equation}
which results in:
\begin{equation}
    \begin{aligned}
    \int_{C(\tau)\leq b} p_{\pi_\beta}d\tau\geq e^{-\epsilon}.\label{eq:assumtion1}
    \end{aligned}
\end{equation}

Next, to solve the optimization problem \eqref{eq:offlinerl_oj_2}\eqref{eq:offlinerl_safe_cons_2}\eqref{eq:offlinerl_traj_cons_2}, we construct the Lagrange function as:
\begin{equation}
    \begin{aligned}
        L(\tau,q,\dot{q}) =& \mathbb{E}_{\tau\sim q(\tau)}[R(\tau)]-\lambda\left( \int_{C(\tau)\leq b}q(\tau)d\tau-1\right)-\mu \left(D_{KL}(q(\tau)||p_{\pi_\beta}(\tau))-\epsilon\right)\\
        =&\int_{\tau\sim q(\tau)}\left[R(\tau)-\mu\left(\log q(\tau)-\log p_{\pi_\beta}(\tau)\right)\right]d\tau-\lambda \int_{C(\tau)\leq b}q(\tau)d\tau+\lambda+\mu\epsilon. 
        \nonumber
    \end{aligned}
\end{equation}
If the inequality~\eqref{eq:assumtion1} holds strictly (i.e., the Slater’s condition holds), we get $\frac{\partial{L}}{\partial q}=0$ from Karush-Kuhn-Tucker (KKT) conditions, i.e.,
\begin{equation}
    \begin{aligned}\label{eq:offline_safe_rl_lag}
        \frac{\partial{L}}{\partial q}=R(\tau)-\mu (\log q(\tau)+1-\log p_{\pi_\beta}(\tau))-\lambda=0.
    \end{aligned}
\end{equation}
Provided that $R(\tau)$ is not a constant function, we have $\mu\neq 0$ in Eq.~\eqref{eq:offline_safe_rl_lag}.
Solving Eq.~\eqref{eq:offline_safe_rl_lag} yields:
\begin{equation}
    \begin{aligned}\label{eq:q*}
        q^*(\tau)=
            \begin{cases} 
            \frac{p_{\pi_\beta}(\tau)\exp(\alpha R(\tau))}{Z} & \text{if }C(\tau)\leq b;\\
            0                           & \text{otherwise},
            \end{cases}
    \end{aligned}
\end{equation}
where $Z=\exp(1+\frac{\lambda}{\mu})$ and $\alpha=1/\mu$.
To exclude the Lagrangian multipliers $\lambda$ and $\mu$ from $Z$ and $\alpha$, we introduce the KKT conditions:
    \begin{align}
        \mu\left(D_{KL}(q(\tau)||p_{\pi_\beta}(\tau))-\epsilon\right)&=0,\label{eq:kkt1}\\
        \int_{C(\tau)\leq b}q(\tau)d\tau=1.\label{eq:kkt2}&
    \end{align}
Solving  Eq.~\eqref{eq:q*}\eqref{eq:kkt1}\eqref{eq:kkt2}, we finally obtain the relationship between $a,Z$ and $\epsilon,b$:
\begin{equation}
    \begin{aligned}
        &Z=\int_{C(\tau)\leq b}p_{\pi_\beta}(\tau)\exp(\alpha R(\tau))d\tau,\\
        &\alpha\int_{C(\tau)\leq b} q^* R(\tau)-\log Z=\epsilon.
        \nonumber
    \end{aligned}
\end{equation}
\end{proof}

\subsection{The Probabilistic Inference Perspective}
\label{app:Probabilistic Inference Perspective}
\paragraph{The derivation of ELBO}
Recall that the objective of probabilistic inference is to minimize the discrepancy between the auxiliary trajectory distribution $q(\tau)$ and the optimal trajectory distribution of the behavior policy $p_{\pi_\beta}(\tau|O=1)$:
\begin{equation}
    \begin{aligned}
        &D_{KL}(q(\tau)||p_{\pi_\beta}(\tau|O=1))\\
        =&\mathbb{E}_{\tau \sim q}[\log q(\tau)]-\mathbb{E}_{\tau \sim q}[\log p_{\pi_\beta}(\tau|O=1)]\\
        =&\mathbb{E}_{\tau \sim q}[\log q(\tau)]-\mathbb{E}_{\tau \sim q}[\log p_{\pi_\beta}(\tau,O=1)]+\mathbb{E}_{\tau \sim q}[\log p_{\pi_\beta}(O=1)]\\
        =&\mathbb{E}_{\tau \sim q}[\log q(\tau)]-\mathbb{E}_{\tau \sim q}\left[\log [p_{\pi_\beta}(O=1|\tau)p_{\pi_\beta}(\tau)]\right]+\log p_{\pi_\beta}(O=1)\\
        =&\mathbb{E}_{\tau \sim q}[\log q(\tau)-\log \left(p_{\pi_\beta}(\tau)\exp(a'R(\tau))/Z'\right)]+\log p_{\pi_\beta}(O=1)\\
        =&-(-D_{KL}(q(\tau)||p_1(\tau)))+\log p_{\pi_\beta}(O=1)\\
        =&-J(q(\tau))+\log P_{\pi_\beta}(O=1),
    \end{aligned}
\end{equation}
where $J(q(\tau))=-D_{KL}(q(\tau)||p_1(\tau))$, $p_1(\tau)=p_{\pi_\beta}(\tau)\exp(\alpha' R(\tau))/Z'$ and $Z'$ is a constant normalizer. 
Since $\log p_{\pi_\beta}(O=1)$ is independent of $q(\tau)$, minimizing the discrepancy between $q(\tau)$ and $p_{\pi_\beta}(\tau|O=1)$ is equivalent to maximizing $J(q(\tau))$.

\paragraph{Trajectory optimization from the probabilistic inference perspective}
Recall the optimization problem derived from the probabilistic inference perspective:
\begin{align}
    \max_{q(\tau)} J(q(\tau))=-D_{KL}(q(\tau)||p_1(\tau))\nonumber\\
    \Pi_b = \{q(\tau):\int_{C(\tau)\leq b}q(\tau)d\tau=1\},\nonumber
\end{align}
where $p_1(\tau)\propto p_{\pi_\beta}(\tau)\exp(\alpha' R(\tau))$.
By directly applying Lemma~\ref{lemma:lemma1}, we replace $\mathcal{A},p(x)$ with $\{\tau|C(\tau)\leq b\},p_1(\tau)$ respectively and obtain 
\begin{equation}
    \begin{aligned}\label{eq:inference_res}
        q^*(\tau)=
            \begin{cases} 
            \frac{p_{\pi_\beta}(\tau)\exp(\alpha' R(\tau))}{Z} & \text{if }C(\tau)\leq b;\\
            0                           & \text{otherwise},
            \end{cases}
            \nonumber
    \end{aligned}
\end{equation}
which has the same form as $q^*$ in Eq.~\eqref{eq:q*} except that $\alpha'$ is a preset hyper-parameter and does not change with budget $b$ and divergence threshold $\epsilon$. 

\paragraph{Difference between $\alpha$ and $\alpha'$}
In Theorem~\ref{thm:q^*_b}, $\alpha$ controls the degree of policy improvement over the behavior policy, and is determined by divergence threshold $\epsilon$ and budget $b$.
However, in Eq.~\eqref{eq:inference_res} which is derived from the probabilistic inference perspective, $\alpha'$ is a preset hyper-parameter independent of $\epsilon$ and $b$. 
This is because the in-distribution constraint does not appear in the probabilistic inference process.
Without the in-distribution constraint, $\alpha'$, which determines the degree of optimality and the distribution's deviation from the behavior policy, can be arbitrarily controlled. 
Since in Theorem~\ref{thm:q^*_b}  the exact dependence of $\alpha$ w.r.t a certain budget is hard to derive analytically, we set $\alpha$ to a fixed value in pratice for different budgets, which can still yield good performances in various tasks, as shown in the empirical results.

\subsection{Proof of Lemma~\ref{lemma:bound} and Proposition~\ref{thm:optimal_guarantee}}
\label{app:proof2}

First, we denote $J_T(\pi)$ as the expected return of $\pi$ under the environment with transition dynamics $T$. Then, we denote $\pi_\beta, \hat T,\hat r,\hat c$ as empirical behavior policy, transition dynamics, reward function and cost function, respectively, which are all estimated from the offline dataset. Let $T$, $r$, $c$ be the transition dynamics, reward and cost in real environment. Let $\pi$ and $T_q$ be the learned policy and transition dynamics induced from the trajectory distribution $q(\tau)$ by:
\begin{equation}    
\begin{aligned}        
    \pi(a|s)&\propto\int q(\tau|\tau_{s_0}=s, \tau_{a_0}=a)d\tau\\        
    T_q(s'|s,a)&\propto\int q(\tau|\tau_{s_0}=s, \tau_{a_0}=a,\tau_{s_1}=s')d\tau.    
\end{aligned}
\end{equation}
For the sake of simplicity in theoretical derivations, we will ignore the sampling error of the reward function and only consider the sampling error of the cost function, which is critical in the constrained problem. 
However, the sampling error of the reward function can also be easily incorporated into the theoretical results.
Besides, we also ignore the difference in initial state distributions $p(s_0)$ across different MDPs and assume the empirical cost function $\hat c$ can be accurately modeled from offline dataset.

\begin{lemma}[4.2 in the main paper]
\label{lemma:bound_2}
    Let $\pi$ be a policy derived from $q(\tau)$.
    Denote $J_{T_q}(\pi)=\mathbb E_{\tau\sim q(\tau)}[R(\tau)]$ as the expected return of $q(\tau)$ and $J_T(\pi)$ as the expected return of $\pi$ under the real environment $T$. Let $C(\tau)$ and $\hat C(\tau)$ be the episodic cost of trajectory $\tau$ under the $T$ and empirical dynamics $\hat T$, respectively.
    With probability at least $1-\delta$, the gap between $J_{T_q}(\pi)$ and $J_T(\pi)$ is bounded by:
    \begin{equation}
        \begin{aligned}
            |J_{T_q}(\pi)- J_{T}(\pi)|\leq  \frac{2R_{m}}{1-\gamma}\bigg[(1-\gamma^{L+1})\sqrt{\frac{L\epsilon}{2}}+\mathbb E_{\pi, \hat T}\left[\sum_{t=0}^L\frac{\gamma^t C_{T,\delta}}{\sqrt{{N}(s_t,a_t)}}\right]\bigg],
            \nonumber
        \end{aligned}
    \end{equation}
    where the first term $(1-\gamma^{L+1})\sqrt{\frac{L\epsilon}{2}}$ vanishes in deterministic environment,
    and for each feasible trajectory $\tau$,
    \begin{equation}
        \begin{aligned}
            |C(\tau)-\hat C(\tau)|\leq \sum_{t=0}^L\frac{\gamma^t C_{c,\delta}}{\sqrt{{N}(\tau_{s_t},\tau_{a_t})}},
            \nonumber
        \end{aligned}
    \end{equation}
   where $C_{T,\delta},C_{c,\delta}$ are constants depending on the concentration properties of $T(s'|s,a),c(s,a)$, respectively, and $\delta\in(0,1)$. $N(s,a)$ is the counts for each state-action pair $(s,a)$ in the offline dataset.
\end{lemma}

Before presenting the proof of Lemma~\ref{lemma:bound_2}, we first introduce the assumption of concentration properties from~\cite{auer2008near,kumar2020conservative} and a modified lemma from~\cite{janner2019trust}. We introduce the assumption of concentration properties for the transition dynamics, and cost function:

\begin{assumption}\label{asm:concentration}
    $\forall s,a\in \mathcal{D}$, the following relationships hold with probability at least $ 1-\delta$:
    \begin{equation}
        \begin{aligned}
            \frac{1}{2}\Vert   T(s'|s,a)-\hat T(s'|s,a) \Vert_1&\leq \frac{C_{T,\delta}}{\sqrt{{N}(s,a)}},\\
            |c(s,a)-\hat c(s,a)|&\leq \frac{C_{c,\delta}}{\sqrt{{N}(s,a)}},
        \end{aligned}
    \end{equation}
    where $C_{T,\delta},C_{c,\delta}$ are constants dependent on $T,c$ respectively and $\delta\in (0,1)$. 
\end{assumption}

Then we provide a lemma modified from Lemma B.3 in \cite{janner2019trust}.
\begin{lemma}\label{lem:return bound}
    The return gap between $\pi_1$ in environment $T_1$ and $\pi_2$ in environment $T_2$ can be bounded as:
    \begin{equation}
        \begin{aligned}
            &|J_{T_1}(\pi)-J_{T_2}(\pi)|\\
            =&\left|\mathbb{E}_{\pi,T_1}[\sum_t \gamma^t r(s,a)]-\mathbb{E}_{\pi,T_2}[\sum_t \gamma^t r(s,a)]\right|\\
            \leq &R_{m}\sum_t \sum_{s,a}\gamma^t|T_1^t(s,a)-T_2^t(s,a)|\\
            \leq &2 R_{m}\sum_t \gamma^t \sum_{t'=0}^t\left(  \mathbb E_{s,a\sim p^{t'}_1(s,a)}\left[D_{TV}(T_1(\cdot|s,a)||T_2(\cdot|s,a))\right]+\mathbb E_{s\sim p^{t'}_1(s)}\left[D_{TV}(\pi_1(\cdot|s)||\pi_2(\cdot|s))\right]\right)
        \end{aligned}
    \end{equation}
    where $p^{t'}_1$ and $p^{t'}_2$ are state/state-action visitation distributions w.r.t $T_1$ and $T_2$, respectively.
\end{lemma}

\begin{lemma}\label{lem:epsilon bound}
    Suppose $\pi_1$ and $T_1$ are induced by trajectory distribution $p_1(\tau)$ and $\pi_2$ and $T_2$ are induced by trajectory distribution $p_2(\tau)$, if $D_{KL}(p_1(\tau)||p_2(\tau))\leq \epsilon$, we can obtain: 
    \begin{equation}
        \begin{aligned}
            \sum_{t}^L\left[ \mathbb E_{s,a\sim p^t_1(s,a)}\left[D_{TV}(p_1(\cdot|s,a)||p_2(\cdot|s,a))\right]\right]&\leq \sqrt{\frac{L\epsilon}{2}}\\
            \sum_{t}^L\left[ \mathbb E_{s\sim p^t_1(s)}\left[D_{TV}(\pi_1(\cdot|s)||\pi_2(\cdot|s))\right]\right]&\leq \sqrt{\frac{L\epsilon}{2}}
        \end{aligned}
    \end{equation}
    where $L$ is the maximum length of trajectory. 
\end{lemma}
\begin{proof}
    \begin{equation}
        \begin{aligned}
            \epsilon &\geq D_{KL}(p_1(\tau)||p_2(\tau))\\
            &=\sum_{\tau} p(\tau) \log\frac{p_1(s_0)\prod_{t=0}^L\pi_1(a_t|s_t)p_1(s_{t+1}|s_t,a_t)}{p_2(s_0)\prod_{t=0}^L\pi_2(a_t|s_t)p_2(s_{t+1}|s_t,a_t)} \\
            &\geq  \sum_{t}^L\mathbb E_{s,a\sim p^t_1(s,a)}\left[D_{KL}(p_1(\cdot|s,a)||p_2(\cdot|s,a))\right]+\sum_{t}^L\mathbb E_{s\sim p^t_1(s)}\left[D_{KL}(\pi_1(\cdot|s)||\pi_2(\cdot|s))\right]\\
            &\geq  \sum_{t}^L\mathbb E_{s,a\sim p^t_1(s,a)}\left[2D_{TV}^2(p_1(\cdot|s,a)||p_2(\cdot|s,a))\right]+\sum_{t}^L\mathbb E_{s\sim p^t_1(s)}\left[2D_{TV}^2(\pi_1(\cdot|s)||\pi_2(\cdot|s))\right]\\
            &\geq  2\sum_{t}^L\left[ \mathbb E_{s,a\sim p^t_1(s,a)}\left[D_{TV}(p_1(\cdot|s,a)||p_2(\cdot|s,a))\right]\right]^2 +2\sum_{t}^L\left[ \mathbb E_{s\sim p^t_1(s)}\left[D_{TV}(\pi_1(\cdot|s)||\pi_2(\cdot|s))\right]\right]^2
        \end{aligned}
    \end{equation}
    Thus,
    \begin{equation}
        \begin{aligned}
            \epsilon &\geq 2\sum_{t}^L\left[ \mathbb E_{s,a\sim p^t_1(s,a)}\left[D_{TV}(p_1(\cdot|s,a)||p_2(\cdot|s,a))\right]\right]^2\\
            &\geq 2L \frac{1}{L}\sum_{t}^L\left[ \mathbb E_{s,a\sim p^t_1(s,a)}\left[D_{TV}(p_1(\cdot|s,a)||p_2(\cdot|s,a))\right]\right]^2\\
            &\geq 2L \left[\frac{1}{L}\sum_{t}^L \mathbb E_{s,a\sim p^t_1(s,a)}\left[D_{TV}(p_1(\cdot|s,a)||p_2(\cdot|s,a))\right]\right]^2\\
        \end{aligned}
    \end{equation}
    Then, we have $\sum_{t}^L\left[ \mathbb E_{s,a\sim p^t_1(s,a)}\left[D_{TV}(p_1(\cdot|s,a)||p_2(\cdot|s,a))\right]\right]\leq \sqrt{\frac{L\epsilon}{2}}$.
    Similarly, we can obtain $\sum_{t}^L\left[ \mathbb E_{s\sim p^t_1(s)}\left[D_{TV}(\pi_1(\cdot|s)||\pi_2(\cdot|s))\right]\right]\leq \sqrt{\frac{L\epsilon}{2}}$.
\end{proof}

Now we are ready to present the proof of Lemma~\ref{lemma:bound_2}:
\begin{proof}
    First, we consider the return gap of $\pi$ between learned dynamics $T_q$ and empirical dynamics $\hat T$.
    As discussed in Section~\ref{sec:rl_to_match}, the $T_q$ and $\hat T$ is consistent in deterministic dynamics, which means
    \begin{equation}
        |J_{T_q}(\pi)- J_{\hat T}(\pi)|=0.
    \end{equation}
    In the case of probabilistic dynamics, since we constraint the trajectory divergence between empirical behavior and target policy ($D_{KL}(q(\tau)||p_{\pi_\beta}(\tau))\leq \epsilon$), we apply \ref{lem:epsilon bound} and \ref{lem:return bound} and get 
    \begin{align}
        &|J_{T_q}(\pi)- J_{\hat T}(\pi)|\nonumber\\
        =&\left|\mathbb{E}_{\pi,T_q}[\sum_{t=0}^L \gamma^t \hat r(s,a)]-\mathbb{E}_{\pi,\hat T}[\sum_{t=0}^L \gamma^t \hat r(s,a)]\right|\nonumber\\
        \leq & 2  R_{m}\sum_{t=0}^L \gamma^t \sum_{t'=0}^t\left(  \mathbb E_{s,a\sim p^{t'}_{\hat T}(s,a)}\left[D_{TV}(T_q(\cdot|s,a)||\hat T(\cdot|s,a))\right]\right)\\
        \leq & 2  R_{m}\sum_{t=0}^L \gamma^t \sqrt{\frac{L\epsilon}{2}}=\frac{2R_m(1-\gamma^{L+1})}{1-\gamma}\sqrt{\frac{L\epsilon}{2}}.
    \end{align}
    

    Then, we consider the return gap of $\pi$ between $\hat T$ and $T$:
    \begin{align}
        &|J_{T}(\pi)- J_{\hat T}(\pi)|\nonumber\\
        =&\left|\mathbb{E}_{\pi,T}[\sum_{t=0}^L \gamma^t r(s,a)]-\mathbb{E}_{\pi,\hat T}[\sum_{t=0}^L \gamma^t r(s,a)]\right|\\ 
        \leq & 2 R_{m}\sum_{t=0}^L \gamma^t \sum_{t'=0}^t\left(  \mathbb E_{s,a\sim p_{\hat T}^{t'}(s,a)}\left[D_{TV}(T(\cdot|s,a)||\hat T(\cdot|s,a))\right]\right)\\
        = &  \frac{2R_{m}}{1-\gamma}\sum_{t'}^L \gamma^{t'}\mathbb E_{s,a\sim p_{\hat T}^{t'}(s,a)}\left[D_{TV}(T(\cdot|s,a)||\hat T(\cdot|s,a))\right]\label{eq:2bound} \\
        \leq & \frac{2R_{m}}{1-\gamma}\mathbb E_{\pi, \hat T}\left[\sum_{t=0}^L\frac{\gamma^t C_{T,\delta}}{\sqrt{{N}(s_t,a_t)}}\right].
    \end{align}
    Thus,
    \begin{align}
        &|J_{T_q}(\pi)- J_{T}(\pi)|\\
        \leq&|J_{T_q}(\pi)- J_{\hat T}(\pi)|+|J_{\hat T}(\pi)- J_{T}(\pi)|\\
        \leq & \frac{2R_m}{1-\gamma}\left[(1-\gamma^{L+1})\sqrt{\frac{L\epsilon}{2}}+\mathbb E_{\pi, \hat T}\left[\sum_{t=0}^L\frac{\gamma^t C_{T,\delta}}{\sqrt{{N}(s_t,a_t)}}\right]\right].
        \nonumber
    \end{align}
    where the first term vanishes in deterministic dynamics.

    Then, we further prove $|J_{T}(\pi)- J_{\hat T}(\pi)|$ can be bound by $\frac{2R_{m}}{1-\gamma} \left[\frac{2(1-\gamma^{L+1})}{1-\gamma}\sqrt{\frac{L\epsilon}{2}}+\mathbb E_{\pi_\beta, \hat T}\left[\sum_{t=0}^L\frac{\gamma^t C_{T,\delta}}{\sqrt{{N}(s_t,a_t)}}\right]\right]$. \qquad\qquad\qquad\qquad 
    After denoting the uncertainty reward $u(s,a)=D_{TV}(T(\cdot|s,a)||\hat T(\cdot|s,a))$ which is less than $1$, the term $\sum_{t'}^L \gamma^{t'}\mathbb E_{s,a\sim p_{\hat T}^{t'}(s,a)}\left[D_{TV}(T(\cdot|s,a)||\hat T(\cdot|s,a))\right]$ can be seen as the cumulative uncertainty reward of $\pi$ and $\hat T$, and we denote it as $U_{\hat T}(\pi)$ for brevity.
    Therefore, we can bound the uncertainty return gap between behavior policy $\pi_\beta$ and learned policy $\pi$:
    \begin{align}
        &|U_{\hat T}(\pi)- U_{\hat T}(\pi_\beta)|\nonumber\\
        =&\left|\mathbb{E}_{\pi,\hat T}[\sum_{t=0}^L \gamma^t u(s_t,a_t)]-\mathbb{E}_{\pi_\beta,\hat T}[\sum_{t=0}^L \gamma^t u(s_t,a_t)]\right|\nonumber\\
        \leq & 2 \sum_{t=0}^L \gamma^t \sum_{t'=0}^t\left( \mathbb E_{s\sim p^t_1(s)}\left[D_{TV}(\pi(\cdot|s)||\pi_\beta(\cdot|s))\right] \right)\\
        \leq & 2 \sum_{t=0}^L \gamma^t \sqrt{\frac{L\epsilon}{2}}=\frac{2(1-\gamma^{L+1})}{1-\gamma}\sqrt{\frac{L\epsilon}{2}}.
    \end{align}
    After upper bounding $U_{\hat T}(\pi)$ by $\frac{2(1-\gamma^{L+1})}{1-\gamma}\sqrt{\frac{L\epsilon}{2}}+U_{\hat T}(\pi_\beta)$, we can obtain the return gap of $\pi$ between $\hat T$ and $T$.
    \begin{align}
        &|J_{\hat T}(\pi)- J_{T}(\pi)|\\
        \leq & \frac{2R_{m}}{1-\gamma} U_{\hat T}(\pi)\\
        \leq &  \frac{2R_{m}}{1-\gamma} \left[\frac{2(1-\gamma^{L+1})}{1-\gamma}\sqrt{\frac{L\epsilon}{2}}+U_{\hat T}(\pi_\beta)\right].
    \end{align}    
    Then, there is:
    \begin{align}
        &|J_{T_q}(\pi)- J_{T}(\pi)|\\
        \leq&|J_{T_q}(\pi)- J_{\hat T}(\pi)|+|J_{\hat T}(\pi)- J_{T}(\pi)|\\
        \leq & \frac{2R_m(1-\gamma^{L+1})}{1-\gamma}\sqrt{\frac{L\epsilon}{2}}+\frac{2R_{m}}{1-\gamma} \left[\frac{2(1-\gamma^{L+1})}{1-\gamma}\sqrt{\frac{L\epsilon}{2}}+\mathbb E_{\pi_\beta, \hat T}\left[\sum_{t=0}^L\frac{\gamma^t C_{T,\delta}}{\sqrt{{N}(s_t,a_t)}}\right]\right].
        \nonumber
    \end{align}
    where the term $\frac{2R_m(1-\gamma^{L+1})}{1-\gamma}\sqrt{\frac{L\epsilon}{2}}$ vanishes in deterministic dynamics.
    For the error bound of the cost, we note that the probability of all constraint-violating trajectories ($\hat C(\tau)> b$) is zero for $\pi$ on both $\hat T$ and $T_q$, since the empirical cost can be accurately modeled even in $T_q$ under our assumption. Therefore, the error only comes from the cost function divergence between the empirical and real environments. For each trajectory $\tau$, we utilize the concentration properties to obtain that
    \begin{align}
        |C(\tau)-\hat C(\tau)|\leq\sum_{t}\gamma^t|c(\tau_{s_t},\tau_{a_t})-\hat c(\tau_{s_t},\tau_{a_t})|\leq \sum_{t=0}^L\frac{\gamma^t C_{c,\delta}}{\sqrt{{N}(\tau_{s_t},\tau_{a_t})}}. \nonumber
    \end{align}
    Finally, we prove the proposition~\ref{lemma:bound_2}.
\end{proof}

\begin{proposition}[4.3 in the main paper]
\label{thm:optimal_guarantee_2}
    Denoting $\pi_{q^*_b}$ as the policy induced by the constrained optimal trajectory distribution $q^*_b$,
    for any policy $\pi$ derived from the trajectory distribution that simultaneously satisfies the in-distribution constraint~\eqref{eq:offlinerl_traj_cons} and safety constraint~\eqref{eq:offlinerl_safe_cons}, the following inequality holds with probability at least $1-\delta$:
    \begin{equation}
        \begin{aligned}\label{eq:return_bound_2}
            J_T(\pi_{q^*_b})\geq  J_T(\pi)-\frac{4R_{m}}{1-\gamma} \bigg[(1-\gamma^{L+1})\sqrt{\frac{L\epsilon}{2}}+\mathbb E_{\pi, \hat T}\left[\sum_{t=0}^L\frac{\gamma^t C_{T,\delta}}{\sqrt{{N}(s_t,a_t)}}\right]\bigg]
        \end{aligned}
    \end{equation}
    where the first term $(1-\gamma^{L+1})\sqrt{\frac{L\epsilon}{2}}$ vanishes in deterministic environment.
    And for each trajectory $\tau$ generated from $q^*_b$, the episodic cost of $\tau$ is bounded with probability at least $1-\delta$ as follow:
    \begin{equation}
        \begin{aligned}\label{eq:cost_bound_2}
            C(\tau)\leq b+\sum_{t=0}^L\frac{\gamma^t C_{c,\delta}}{\sqrt{{N}(\tau_{s_t},\tau_{a_t})}}.
        \end{aligned}
    \end{equation}
\end{proposition}

\begin{proof}
Let $\sigma_1=\frac{2R_m}{1-\gamma}\left[(1-\gamma^{L+1})\sqrt{\frac{L\epsilon}{2}}+\mathbb E_{\pi, \hat T}\left[\sum_{t=0}^L\frac{\gamma^t C_{T,\delta}}{\sqrt{{N}(s_t,a_t)}}\right]\right]$. Recall that $q_b^*$ is the optimal distribution, while $\pi_{q_b^*}$ enjoys the theoretical property of $q_b^*$. Hence, for any $\pi$ derived from the trajectory distribution that satisfies both the in-distribution constraint and safety constraint, we have $J_{T_q}(\pi_{q_b^*})\geq  J_{T_q}(\pi)$. Then we utilize the Lemma~\ref{lemma:bound_2} to obtain
\begin{equation}
    \begin{aligned}
        J_T(\pi_{q^*_b})\geq & J_{T_q}(\pi_{q^*_b})-\sigma_1\geq J_{T_q}(\pi) - \sigma_1  \geq J_T(\pi) - 2\sigma_1 \\
        \geq & J_T(\pi) - \frac{4R_m}{1-\gamma}\left[(1-\gamma^{L+1})\sqrt{\frac{L\epsilon}{2}}+\mathbb E_{\pi, \hat T}\left[\sum_{t=0}^L\frac{\gamma^t C_{T,\delta}}{\sqrt{{N}(s_t,a_t)}}\right]\right].
        \nonumber
    \end{aligned}
\end{equation}
Besides, $\hat C(\tau)\leq b$ is true for any trajectory $\tau$ generated from $q^*_b$, thus, by using the Lemma~\ref{lemma:bound_2} we can obtain 
\begin{equation}
    \begin{aligned}
        C(\tau)\leq \hat C(\tau)+\sum_{t=0}^L\frac{\gamma^t C_{c,\delta}}{\sqrt{{N}(\tau_{s_t},\tau_{a_t})}}\leq b+\sum_{t=0}^L\frac{\gamma^t C_{c,\delta}}{\sqrt{{N}(\tau_{s_t},\tau_{a_t})}}.
        \nonumber
    \end{aligned}
\end{equation}

\end{proof}

\section{Details of Simulation Environments and Datasets}
\label{app:environment_detail}
\paragraph{Pendulum and Reacher}
In the Pendulum swing-up task, the agent tries to keep the pendulum upright and balanced under the constraint of keeping away from unsafe angles. 
In the Reacher task, the robotic arm needs to move towards the goal while avoiding the unsafe region. 
The implementation of the Pendulum swing-up task follows Sauté~\cite{sootla2022saute} without additional modification.
The single Pendulum swing-up is taken from the classic control library in the Open AI Gym~\cite{brockman2016openai}.
The cost function is defined as:
\begin{equation}
    \begin{aligned}
        c=\begin{cases}
            1-\frac{|\theta-\delta|}{50} &\text{if }-25\leq \theta \leq 75,\\
            0 & \text{otherwise},
        \end{cases}
        \nonumber
    \end{aligned}
\end{equation}
where $\theta$ is the angle of the pole deviation from the upright position and $\delta=25$. Such a cost function is designed to create a trade-off between keeping upright and staying away from the unsafe degree i.e., $\delta$. 
The cost function of Reacher is designed as:
\begin{equation}
    \begin{aligned}
        c=\begin{cases}
            100-200\Vert x-x_{\text{unsafe}}\Vert_2 & \text{if }\Vert x-x_{\text{unsafe}}\Vert_2 \leq 0.5\\
            0 & \text{otherwise},
        \end{cases}
        \nonumber
    \end{aligned}
\end{equation}
where $x_{\text{unsafe}}$ is set to $(0.5,0.5,0)$, following the implementation of Sauté.
Refer to~\cite{sootla2022saute} for more details.
To collect the offline dataset, we train an unconstrained SAC policy online until convergence and take its replay buffer as the offline dataset, which contains $2e4$ samples.

\paragraph{MuJoCo Environments}
In MuJoCo tasks~\cite{todorov2012mujoco}, following the constraint setting of previous works~\cite{zhang2020first,yang2022cup}, the goal of the agent is to control the robot to keep it balanced and go forward with the cost of the moving speed. 
We use the MuJoCo environments provided by OpenAI Gym~\cite{brockman2016openai}
and refer to FOCOPS~\cite{zhang2020first} to design the cost function: $c=|v|$, where $v$ is the agent's velocity.  
We utilize three types of offline datasets (medium, medium-replay and medium version )  from D4RL~\cite{fu2020d4rl} for each environment. 

\paragraph{Bullet-Safety-Gym Environments}
Bullet-Safety-Gym~\cite{gronauer2022bullet} contains a set of safe RL tasks that are widely used to assess safe RL algorithms.
In the SafetyCarCircle task, the goal of a car agent is to move on a circle while remaining in the safety zone. 
The cost in this task is defined as $c=\mathbb{I}(|x|>x_\text{lim})$. 
In the SafetyBallReach task, a ball agent is controlled to move towards a goal while avoiding collisions with obstacles. 
To collect the offline dataset, we train an unconstrained SAC policy online for $1e6$ steps (one training step for one data point) and take its replay buffer as the offline dataset, which contains $1e6$ samples.

\paragraph{Relationship between reward and cost}
As shown by our experimental results, the reward drops when enforcing stricter constraints.
This is because there exists various degrees of conflict between the task objective and the safety constraint in the above environments (e.g., in MuJoCo tasks, the cost function is also an important component of reward function). 
As a result, the optimal solution would probably be located at some edge area of the feasible policy space. 
When the constraint is set stricter, the edge of feasible space will be pushed inward and the reward of the optimal policy is likely to shrink. 
In other words, these tasks have a non-empty tempting policy class which lies in outside of the feasible policy space and has higher reward than the optimal policy, as formally defined as Temptation in the mentioned work~\cite{liu2022robustness}.

\section{Real-world Advertising Bidding Scenario}
\label{app:bid}
\paragraph{Problem Formulation}
In the real-world oCPM advertising bidding scenario, advertiser $s$ provides the target ROI $Roi_s$ and the maximum total pay $P_{\max}$ at the beginning of each day.
At time step $t$, the advertising system bids $pay_{s,t}$ for advertiser $s$.
If this bid is successful, the advertiser expends $pay_{s,t}$ and obtains revenue $gmv_{s,t}$, i.e., the gross merchandise value, from the displayed ads.
In this process, the system has to ensure that the average conversion is not less than the target $Roi_s$, i.e.,
\begin{equation}
    \begin{aligned}
        \frac{\sum_{t=0}^T gmv_{s,t}}{P_{\max}} > Roi_s\iff \sum_{t=0}^T -gmv_{s,t}< -Roi_s\cdot P_{\max}.
    \end{aligned}
\end{equation}
The optimization problem in this scenario can be described as:
\begin{equation}
    \begin{aligned}
        &\max \sum_{t=0}^{T} pay_{s,t}\\
        \text{s.t. }&\sum_{t=0}^T -gmv_{s,t}< -Roi_s\cdot P_{\max},\forall s
    \end{aligned}
\end{equation}

We formulate this problem as a Markov Decision Process (MDP). 
\begin{itemize}
\item A trajectory contains the record of one advertiser's successful bidding within one day.
\item The state contains the user context features, ads context features, cumulative revenue, cumulative pay, maximum total pay, target ROI, predicted click-through rate $pCTR$, predicted conversion rate $pCVR$ and predicted gross merchandise volume $pGMV$ of one advertiser.
\item The action is to output a weight $k\in [0.5,1.5]$, which controls the bidding by formula $bid=k\cdot \frac{pCTR\cdot pCVR\cdot pGMV}{Roi_s}$.
\item The reward is $pay_{s,t}$ for a successful bid and $0$ otherwise.
\item The cost is $-gmv_{s,t}$ for a successful bid and $0$ otherwise.
\item The constraint budget is $-Roi_s\cdot P_{\max}$.
\end{itemize}

\paragraph{Deployment Details}
In this task, the cost and the constraint budget are negative, which means that the constraint is violated at the starting stage of bidding and can only be satisfied over time.
Although Sauté assumes that the cost is positive, it still works in this scenario.
We modify Sauté by augmenting the state with the value $Roi_s\cdot P_{\max}-\sum_{t=0}^T gmv_{s,t}$ and adding an extra high reward to the state with $Roi_s\cdot P_{\max}-\sum_{t=0}^T gmv_{s,t}<0$.
In this way, the agent is encouraged to reach the target cumulative $gmv$ and maximize the total pay.
Besides, during evaluation, each advertiser switches among policies trained by different algorithms, in order to eliminate the differences between advertisers and make the online experiment fairer.

\paragraph{About the Data-collecting Policy}
The  dataset used in our experiment is collected by a set of constrained policies jointly.
The policy set includes BCORLE($\lambda$)~\cite{zhang2021bcorle} (60\% advertisers), PID~\cite{stooke2020responsive} (20\% advertisers) and CEM~\cite{de2005tutorial} (20\% advertisers).


\section{Implementation Details}
\paragraph{Implementation Details of TREBI}\label{app:trebi_imple}

The overall implementation of TREBI is based on Diffuser~\cite{janner2022planning} and can be found at \url{https://github.com/qianlin04/Safe-offline-RL-with-diffusion-model}.
The inference of TREBI requires three components that should be obtained during the training stage.
\textbf{The first component} is a U-Net style neural network, which mainly consists of one-dimensional temporal convolutions.
It models the mean $\mu_\theta(\tau^i,i)$ of Gaussian distribution $p_\theta=\mathcal{N}(\tau^{i-1};\mu_\theta(\tau^i,i),\Sigma)$ in the reverse process.
The input of this network is the trajectory $\tau^{i-1}$ obtained in the previous denoising steps, which can be seen as a special ``image'' where 1) the length of the trajectory is analogous to the size of the image and 2) the dimension of the concatenation of the state and the action is analogous to the channel of the image.
\textbf{The second and third components} are respectively the reward and the cost value estimation, i.e., $R(\tau)$ and $C(\tau)$, which form the objective function $h_{b,n}(\tau)$ in Eq.~\eqref{eq:pn=phn}.
They are modeled by neural networks that take $\tau^{i-1}$ as the input and output the value estimation.
These two components also contain U-Net style networks similar to the first component.
The loss of the reward/cost value estimation is the MSE loss between the prediction of a trajectory sampled from the dataset and the cumulative reward/cost of this trajectory.
In other words, the value estimation is learned by using the Monte Carlo method.
In addition, during planning multiple trajectories rather than one single trajectory are generated randomly in the reverse process, and the one with the highest cumulative reward among constraint-satisfying trajectories is selected to generate the action for real-environment execution.

\paragraph{Implementation Details of Baselines}

As for the baselines, we implement BCQ-Lagrangian, BCORLE($\lambda$) and BCQ-Sauté by introducing the Lagrange multiplier and the two kinds of state augmentations respectively on the open-source BCQ implementation.
Since there is currently no publicly available CPQ implementation, we reproduce CPQ and follow the hyper-parameters setting mentioned in paper~\cite{xu2022constraints}, unless the learning rate of $\alpha$ is set to 0.01 and the latent space threshold is set to the fixed percentile 0.75 of the latent KL loss of the whole batch dataset.
For BCORLE($\lambda$), the maximum lambda is set to 20 for all tasks.
For BCQ-Sauté, the maximum of initial budget is set to the maximum budget w.r.t. for each task and the penalty for the states violating the constraint is set to 1000.
For all actor-critic methods, the learning rate is $1e-5$ for the actor and $1e-3$ for the critic.
For all the BCQ-based methods, we fine-tune the perturbation range $\Phi$ among $\{0.15,0.10,0.05,0.015\}$.

\paragraph{Hyper-parameter Settings}
\label{app:hyper-parameters}
In all simulation environments, we evaluate each algorithm on 60 different trajectories for each possible budget.
All algorithms are trained for $1e6$ iterations and the results are obtained on 3 different random seeds.
The time interval of regenerating the trajectory for decision making (i.e., control frequency) is set to $1$ for all tasks.
The cost discount factor is set to $\gamma=1$ for Pendulum, Reacher and two Bullet-Safety-Gym task, and $\gamma=0.99$ for all MuJoCo tasks.
The max episode length is set to $200$ for Pendulum, $50$ for Reacher, $1000$ for all MuJoCo tasks, $500$ for SafetyCarCircle-v0 and $250$ for SafetyBallReach-v0.
For Pendulum, Reacher and two Bullet-Safety-Gym tasks, the maximum budget is set to the expected episodic cost of the unconstrained behavior policy. 
For MuJoCo tasks, the maximum budget is set to the expected episodic cost induced by Diffuser's unconstrained planning.
We provide the maximum budgets for all environments in Table~\ref{tab:max_budget}. 
The hyper-parameter $n$ in Eq.~\eqref{eq:gb} is set to $1000$ for Pendulum, Reacher, the Ads bidding task and $100$ for all MuJoCo tasks and Bullet-Safety-Gym tasks.
The length of trajectory is set to $128$ for Pendulum and the Ads bidding task, and $32$ for Reacher, all the MuJoCo tasks and Bullet-Safety-Gym tasks.
The hyper-parameter $\alpha$ in Eq.~\eqref{eq:gb} is set to $0.1$ for all tasks.

\begin{table}[t]
\caption{Maximum budgets for the used datasets.}
\label{tab:max_budget}
\begin{center}
\begin{small}
\begin{tabular}{lr}
\toprule
Dataset & Max budget\\
\midrule
Pendulum & 60 \\
Reacher & 14 \\
HalfCheetah-medium-v2 & 414 \\
HalfCheetah-medium-replay-v2 & 384 \\
HalfCheetah-medium-expert-v2 & 731 \\
Hopper-medium-v2 & 130 \\
Hopper-medium-replay-v2 & 140 \\
Hopper-medium-expert-v2 & 169 \\
Walker2d-medium-v2 & 160 \\
Walker2d-medium-replay-v2 & 160 \\
Walker2d-medium-expert-v2 & 234 \\
SafetyCarCircle-v0 & 150\\
SafetyBallReach-v0 & 35\\
\bottomrule
\end{tabular}
\end{small}
\end{center}
\end{table}

\section{Additional Results}

\paragraph{Results on Hopper and Walker2d}
\label{app:full_mujoco_results}

\begin{figure*}[ht!]
\begin{center}
\centerline{\includegraphics[width=\textwidth]{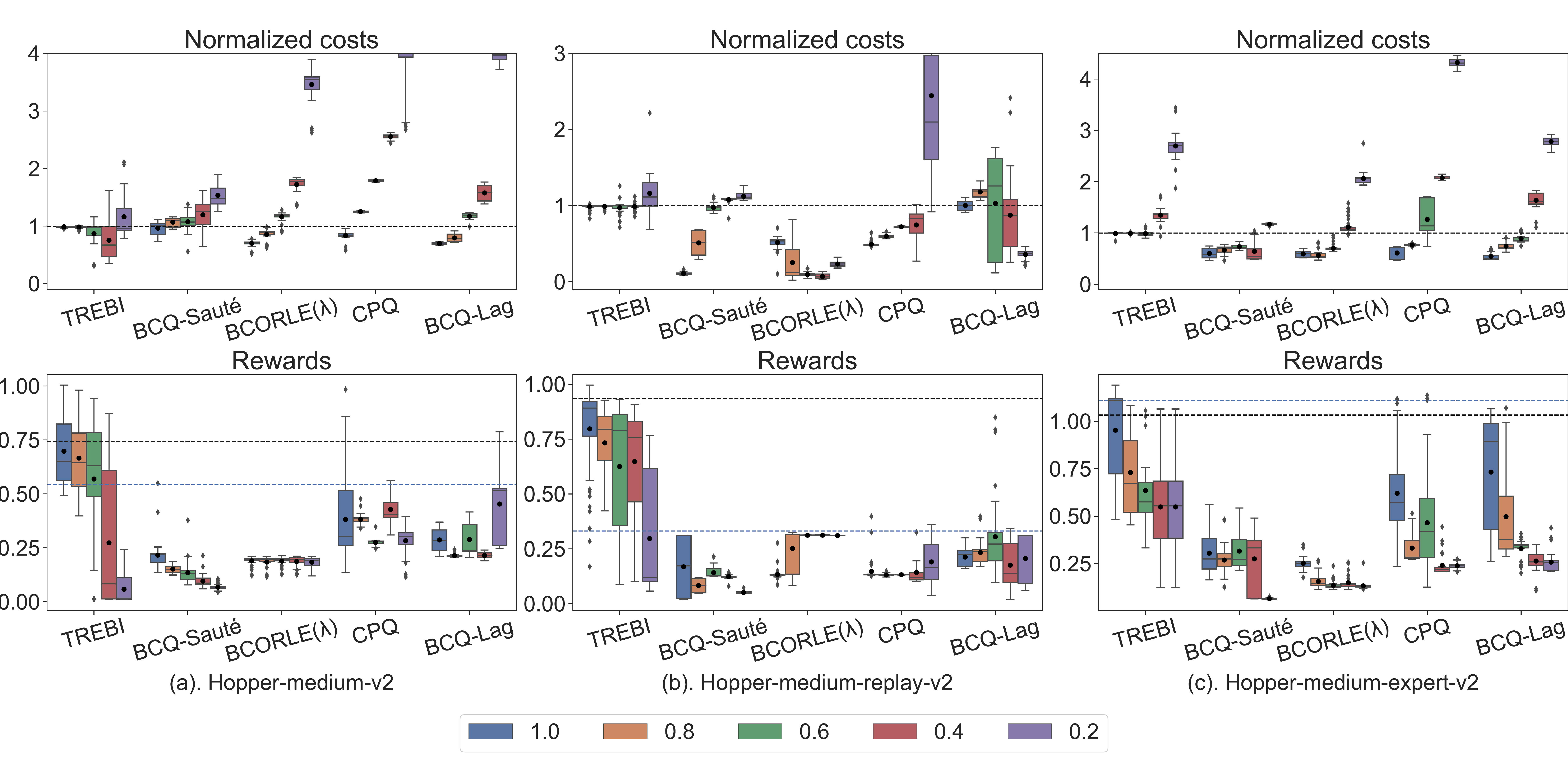}}
\caption{
    Results on Hopper tasks.
} \label{fig:hopper}
\end{center}
\end{figure*}

\begin{figure*}[ht!]
\begin{center}
\centerline{\includegraphics[width=\textwidth]{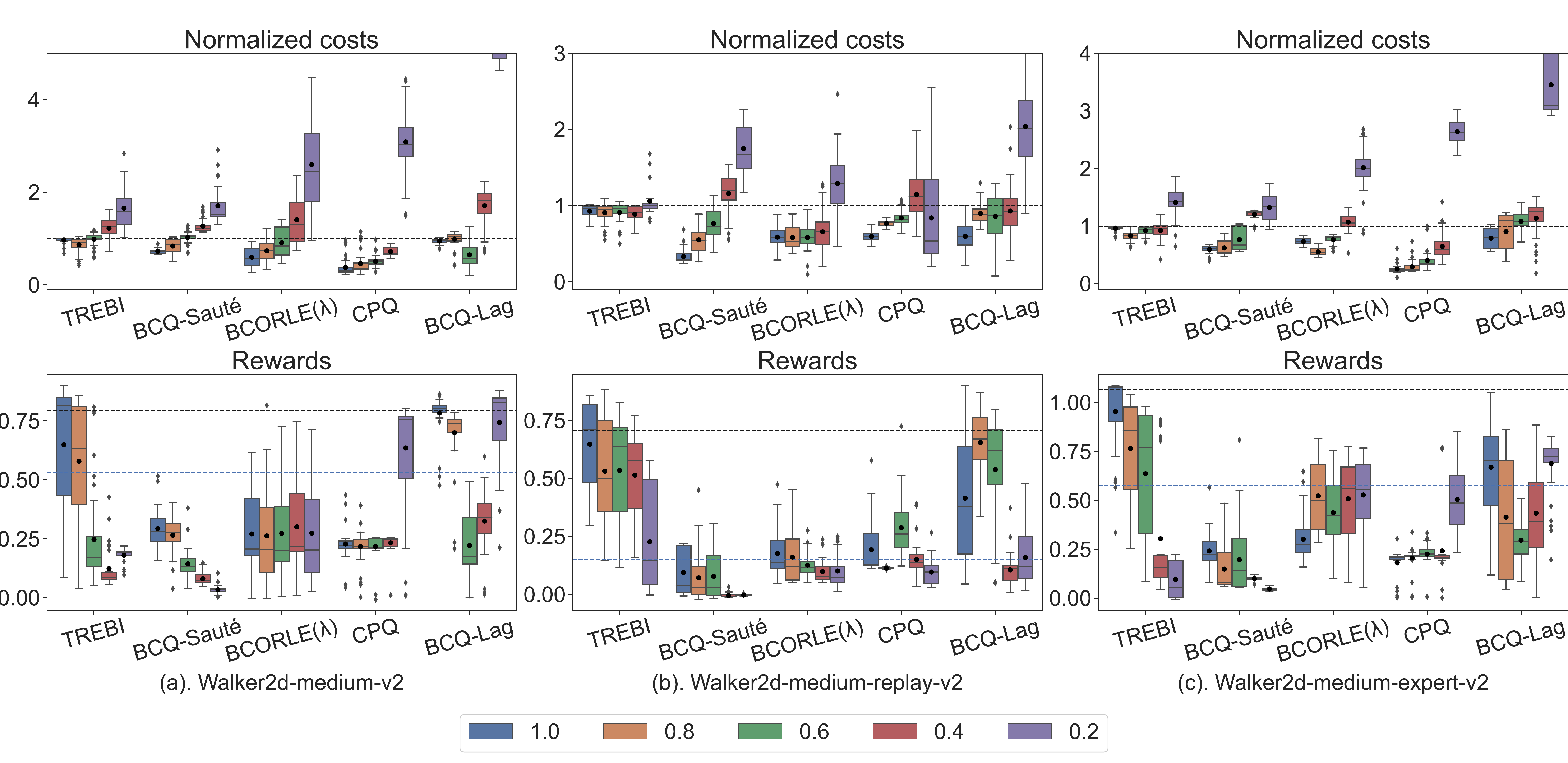}}
\caption{
    Results on Walker2d tasks.
} \label{fig:walker2d}
\end{center}
\end{figure*}

We present the performance of TREBI on Hopper and Walker2d datasets in Figure~\ref{fig:hopper} and \ref{fig:walker2d} respectively.
For all MuJoCo tasks with a budget ratio of $1.0,0.8,0.6$, TREBI achieves nearly perfect constraint satisfaction in the trajectory level.
In some tasks with a low budget ratio of $0.4, 0.2$, none of the evaluated approaches can achieve constraint satisfaction even in the sense of expectation.
One possible cause can be that the datasets of these tasks lack the trajectory samples that satisfy such extremely low budget constraints, hence satisfying the in-distribution constraint makes it almost impossible to learn safe policies for these low budgets.
Besides, TREBI demonstrates better low-budget constraint satisfaction on medium-replay datasets rather than medium and medium-expert datasets.
Note that the overall returns of the three types of datasets satisfies: medium-replay $<$ medium $<$ medium-expert.
While the reward and cost in MuJoCo tasks (and actually in many real-world tasks) have an approximately positive correlation, there are less low-cost trajectories in the medium and medium-expert datasets than in the medium-replay datasets, leading to a difficulty to learn safe policies for low-budget constraints.
Therefore, these empirical results agree with the condition of the dataset's quality in Theorem~\ref{thm:q^*_b}.

\paragraph{Control Frequency}
\label{app:control_freq}
\begin{figure*}[ht!]
\begin{center}
\centerline{\includegraphics[width=0.5\textwidth]{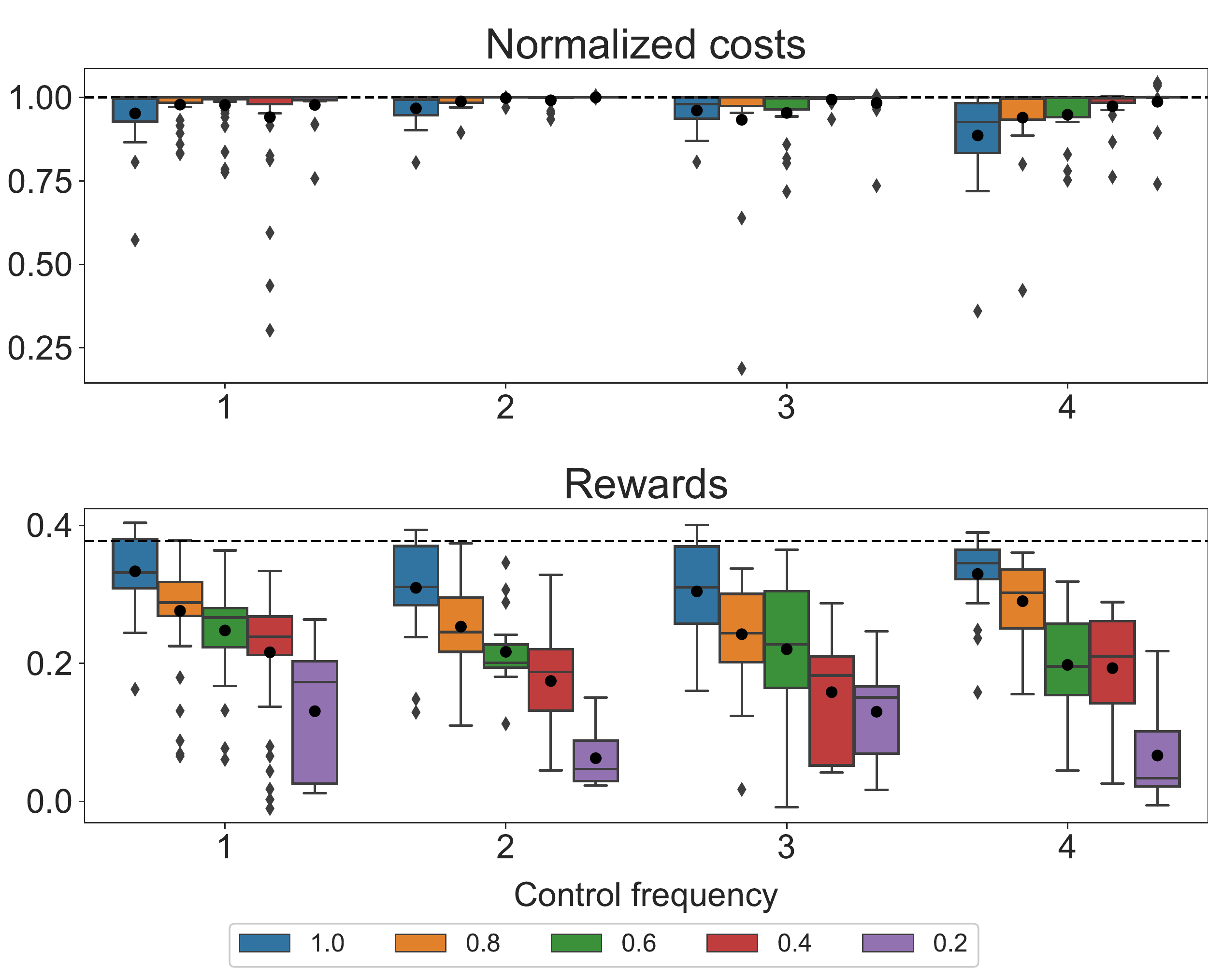}}
\caption{
    Results on HalfCheetah-medium-replay tasks with different constrol frequency.
} \label{fig:control_freq}
\end{center}
\end{figure*}
In our experiments, the planning process is performed at each step to ensure the best performance. 
Here in Figure~\ref{fig:control_freq}, we provide an empirical analysis of the control frequency (i.e., the number of step intervals for replanning)  in halfcheetah-medium-replay. 
We can see that decreasing the control frequency from per-1-step to per-4-step leads to slight reward declines and insignificant constraint violations in most budget settings. 
Thus it is possible to reduce the overall computational costs by decreasing the control frequency without too much performance degradation and constraint violation.

\paragraph{Budgets ratio sampled from $[0,1]$}
\label{app:uniform_sampled_budget}

To further validate the effectiveness of TREBI under randomly selected budgets, we do additional experiments in MuJoCo tasks and Bullet-Safety-Gym tasks where the ratios to maximum budgets are uniformly sampled from interval $[0,1]$. 
For each algorithm, we sample 100 safety budgets and demonstrate the performances under these budgets. 
Note that the fixed-budget baselines BCQ-Lagrangian and CPQ are not evaluated under this setting since they have to retrain for each new budget, which  leads to an enormous amount of computation. 
The results in  Figures~\ref{fig:halfcheetah_RS}-\ref{fig:BulletSafetyGym_RS} demonstrate that TREBI performs consistently well across different levels of budgets in various environments. 

\begin{figure*}[htbp]
  \centering
  \begin{minipage}[t]{1.0\linewidth}  
        \centerline{\includegraphics[width=\textwidth]{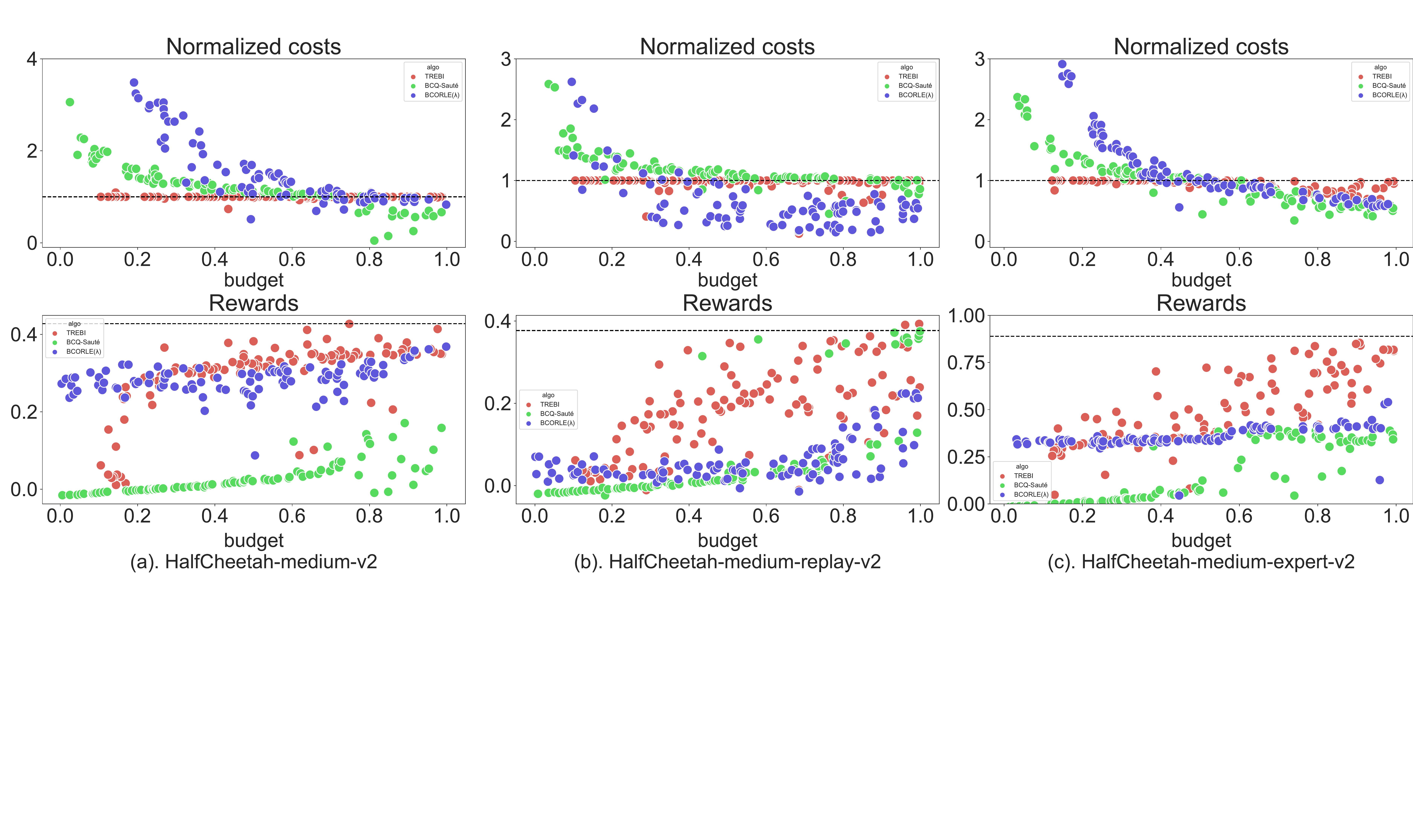}}
        \vskip -1.3in
        \caption{
            Results on HalfCheetah tasks with random budgets sampled from $[0,1]$. 
            The $x$-axis of a sample represents the given budget, which is set to $x$ times the maximum budget, and the $y$-axis represents the cost and reward under the given budget.
        } \label{fig:halfcheetah_RS}
  \end{minipage}
  \vskip -0.8in
  \begin{minipage}[t]{1.0\linewidth}
        \centerline{\includegraphics[width=\textwidth]{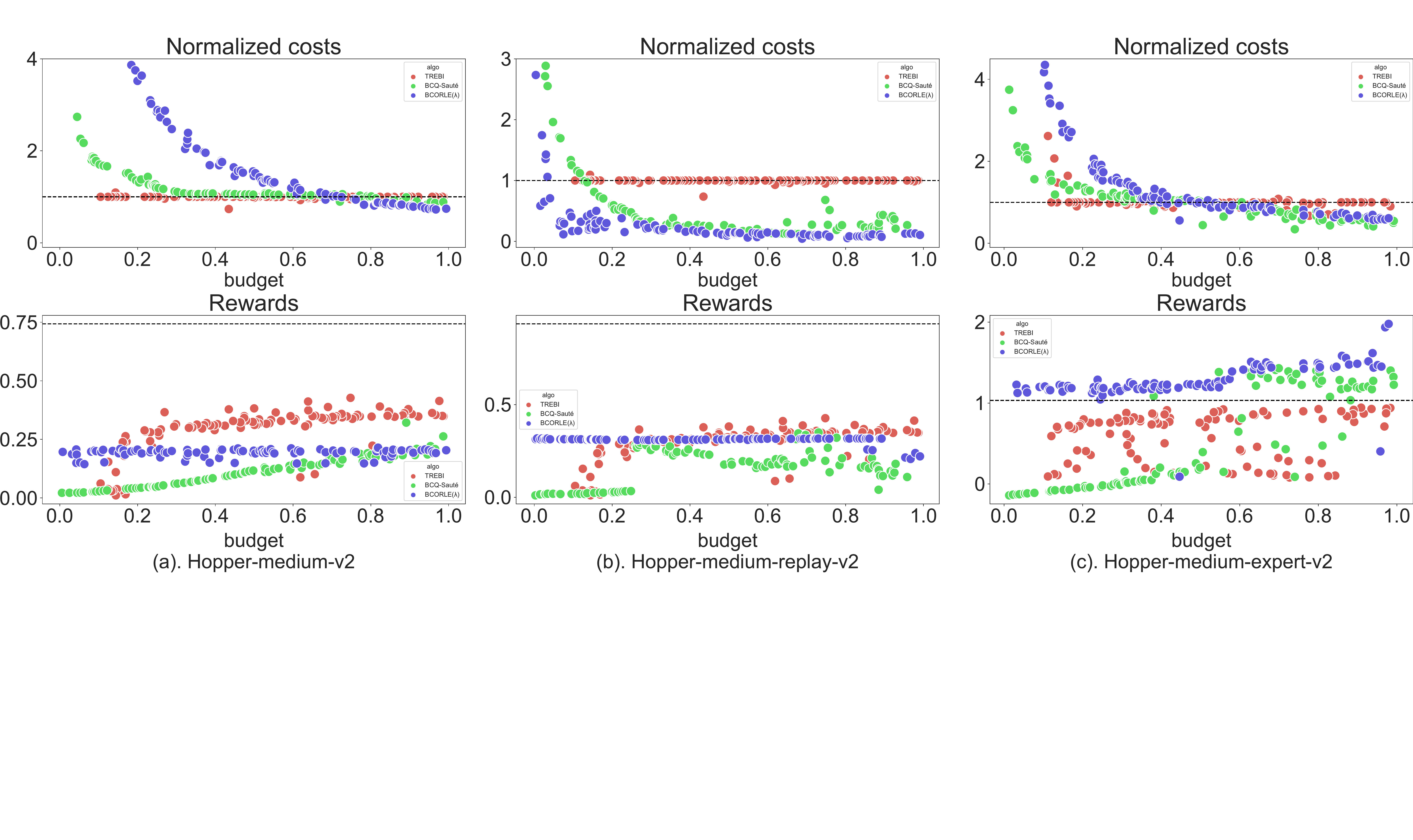}}
        \vskip -1.3in
        \caption{
            Results on Hopper tasks with random budgets sampled from $[0,1]$.
        } \label{fig:hopper_RS}
  \end{minipage}
\end{figure*}

\begin{figure*}[htbp]
    \centering
    \begin{minipage}[t]{1.0\linewidth}
        \centerline{\includegraphics[width=\textwidth]{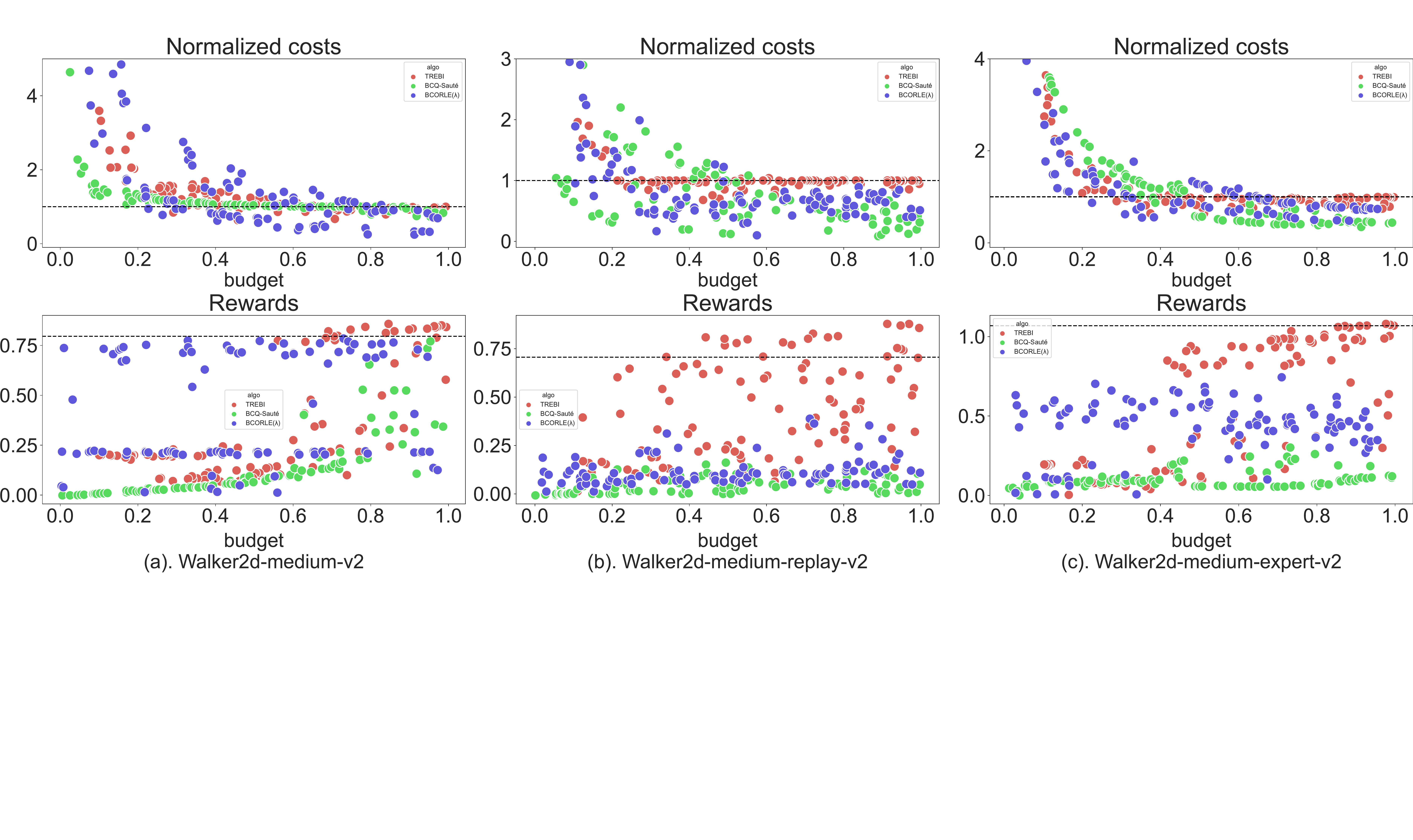}}
        \vskip -1.3in
        \caption{
            Results on Walker2d tasks with random budgets sampled from $[0,1]$.
        } \label{fig:walker2d_RS}
    \end{minipage}
    \vskip -0.8in
    \begin{minipage}[t]{0.66\linewidth}
        \centerline{\includegraphics[width=\textwidth]{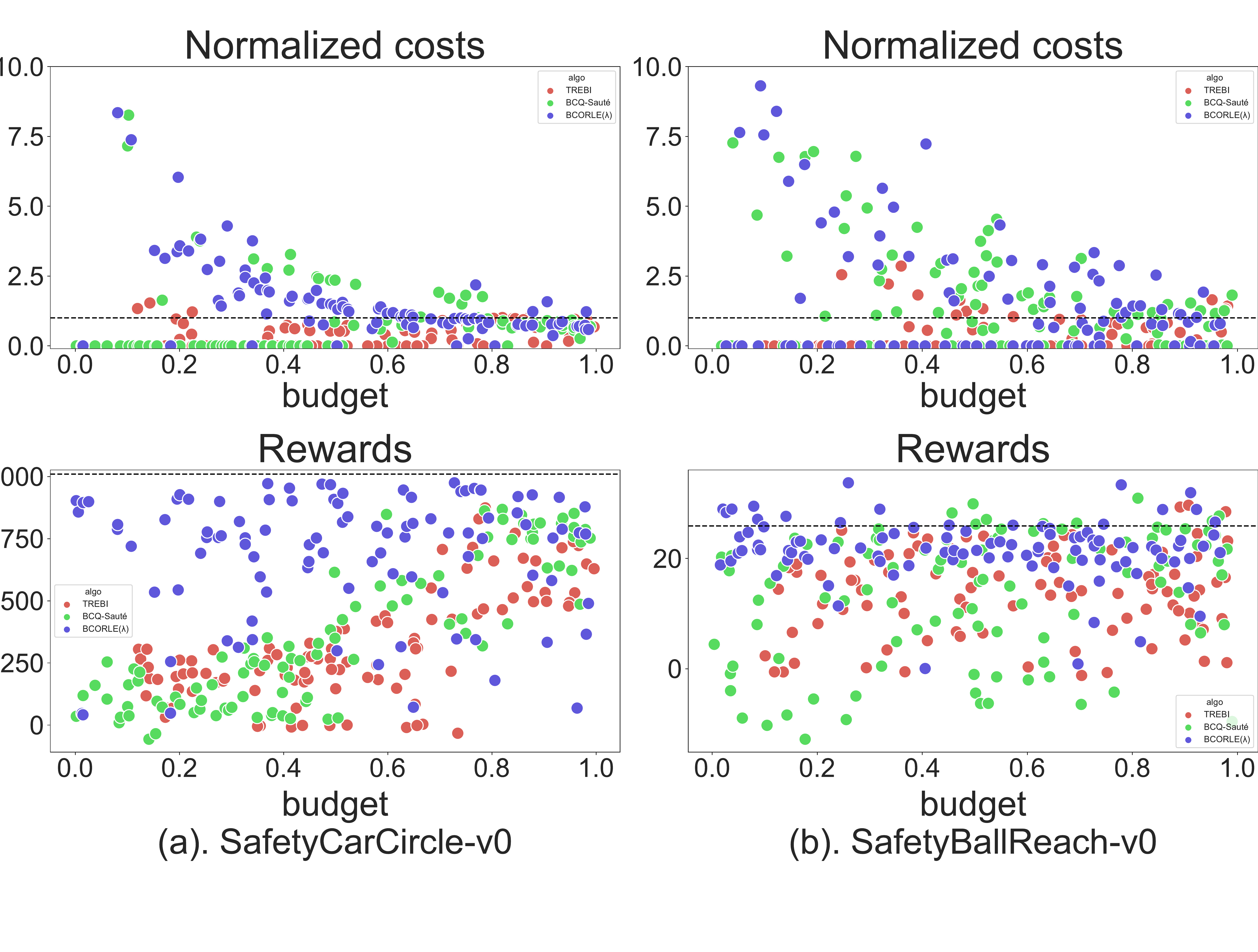}}
        \caption{
            Results on Bullet-Safety-Gym tasks with random budgets sampled from $[0,1]$.
        } \label{fig:BulletSafetyGym_RS}
    \end{minipage}
\end{figure*}

\end{document}